\documentclass{article}

\usepackage{our-style}
\usepackage{natbib}
\usepackage{authblk}
\usepackage[left=3.81cm, right=3.81cm]{geometry}

\usepackage[utf8]{inputenc} 
\usepackage[T1]{fontenc}    
\usepackage{hyperref}       
\hypersetup{
    colorlinks=true,
    citecolor=darkgreen
}
\usepackage{url}            
\usepackage{booktabs}       
\usepackage{amsfonts}       
\usepackage{nicefrac}       
\usepackage{microtype}      
\usepackage{xcolor}         
\usepackage{listings}
\usepackage{subcaption}
\usepackage{float}
\usepackage{tcolorbox}
\tcbset{
    parbox=false,
    colback=gray!10!white,        
    colframe=gray!10!white,       
    boxrule=0mm,                  
    sharp corners,                
    width=\textwidth,             
    before=\vspace{0mm},          
    after=\vspace{0mm},           
    boxsep=1mm,                   
    left=0.5mm,                     
    right=0.5mm,                    
    top=1mm,                      
    bottom=1mm                    
}

\title{Catastrophic Compositional Generation: Why Vanilla Diffusion Models Fail to Extrapolate}

\author[1]{Duncan Soiffer\thanks{Corresponding author: dsoiffer@cs.cmu.edu}}
\author[1]{Chandler Squires}
\author[1]{Yuan Guan}
\author[2, 3]{Jason Hartford}
\author[1]{Pradeep Ravikumar}
\affil[1]{Machine Learning Department, Carnegie Mellon University}
\affil[2]{Valence Labs}
\affil[3]{Department of Computer Science, University of Manchester}
\date{}

\usepackage{changepage}
\begin{document}

\maketitle

\begin{adjustwidth}{0.145cm}{0.145cm} 
\begin{abstract}
    The task of \emph{compositional generation} involves using a conditional generative model, trained only on a subset of the possible conditions, to produce samples from compositionally-defined target distributions such as a geometric combination of the source distributions.
    In this work, we argue that this task is often infeasible for vanilla conditional diffusion models: we conjecture that no inference-time technique can efficiently produce samples from the target distribution in certain well-motivated settings.
    This idea is supported by theory-guided generalization arguments and carefully-designed experiments on both synthetic and realistic data.
    In particular, while recent methods such as Feynman-Kac correction reduce \emph{inference-time approximation error}, our results show that \emph{score estimation error} has a more catastrophic effect on performance when the target distribution is out-of-distribution with respect to the sources, highlighting the need for a different approach to this task.
\end{abstract}
\end{adjustwidth}

\addtocontents{toc}{\protect\setcounter{tocdepth}{0}}

\section{Introduction}\label{sec:intro}

The space of distributions that we would like to model is often far larger than the set of distributions to which we have access. 
We want models that can imagine arbitrary combinations of concepts (e.g. ``a \emph{living room} with a \emph{white couch}, a \emph{black chair}, two \emph{paintings}, a \emph{floor lamp} and nothing else''), but the data are only supported for some combinations of these concepts.
This is particularly true when data are derived from experiments, e.g., in biology, the space of possible experiments is combinatorially large (covering all possible combinations of molecules, gene knockouts, cell types, assays, etc.), but experimental budgets are finite.
Hence, there is significant interest in building perturbation effect prediction models that can predict the outcomes of unseen experiments \citep{lotfollahi2023cpa, roohani2024gears, wang2024sclambda, adduri2025predicting, wenkel2026txpert, bunne2024virtualcell, noutahi2025virtualcells}.

To extrapolate to new combinations, methods for \emph{compositional generation} rely on strong inductive biases, e.g., assuming the existence of some latent space in which the effects have some additive structure.
For example, if two biological perturbations $a$ and $a'$ are \emph{causally separable}, it can be shown that the double perturbation distribution $P(x \mid \DO(a), \DO(a'))$ can be expressed as a geometric combination of the control distribution $P(x)$ and the single perturbation distributions $P(x \mid \DO(a))$ and $P(x \mid \DO(a'))$ \citep{wang2023concept, xu2024}.
In principle, such results imply significant reductions in the amount of training data necessary to succeed at these difficult generation tasks.

In practice, however, even if the compositional distribution $P(x \mid \DO(a), \DO(a'))$ can be expressed in terms of distributions from which we have training samples, this does not imply that we can efficiently sample from $P(x \mid \DO(a), \DO(a'))$. 
Unfortunately, distribution-level identities do not necessarily translate into computationally and statistically efficient algorithms.
In particular, we consider procedures for composing 
\emph{conditional diffusion models}, given their practical relevance and powerful generative capabilities \citep{dhariwal2021diffusion}.
Notably, geometric combinations at the distribution level translate into linear combinations at the score level.
Hence, a popular heuristic for compositional generation with such models is to linearly combine the scores \citep{liu2022compositional} during the denoising process.
However, it is increasingly recognized that this na{\"i}ve approach introduces a fundamental source of error, even assuming that the models achieve perfect recovery: in general, adding noise to the distributions breaks the linear relationship between their score functions.

To account for this issue, recent particle-based algorithms \citep{skreta2025feynman, xie2026enhanced, ren2025driftlite} explicitly correct for this error. 
We focus on Feynman-Kac Correctors (FKC) \citep{skreta2025feynman}, a flexible instantiation of this approach that tracks importance weights across the denoising trajectory and uses them for correction.
Despite these recent advances, we highlight an issue which poses even more significant problems for composition. We find that in many settings of interest, the \emph{score estimation error}, i.e., the error propagated from errors in the model's learned score function, becomes the dominant factor, since the typical samples from the composed distribution are in low-density regions of the source distributions. 
Moreover, in these cases, we find that FKC reinforces these errors, further degenerating sample quality.



\paragraph{Contributions} Building on prior work, we first introduce a \emph{concrete formulation of the compositional generation task}, show that such compositions are \emph{closely connected to causal representation learning}, and \emph{formalize two sources of error} that make this task difficult (\Cref{sec:setup} and \Cref{sec:perspectives}).
Then, we provide \emph{theoretical insights into how these errors behave} for different collections of source distributions (\Cref{sec:theoretical-insights}), focusing on illustrative, analytically tractable settings.
Based on these insights, we perform careful experimentation to disentangle score estimation error from inference-time approximation error (\Cref{sec:experiments}). These experiments reveal that both sources of error lead to compositional failure within distinct, partially overlapping regimes, and that within the overlap, score estimation error often dominates attempts to correct for inference-time approximation error.

\section{Setup: Geometrically weighted compositions}\label{sec:setup}


Mathematically, we consider a \defword{perturbation space} $\cA$ and an \defword{outcome space} $\cX$.
For example, if a scientist can apply $K$ drugs at varying dosages, and measures cell images after such perturbations, we would have $\cA = \bbR^K_{\geq 0}$ and we would take $\cX$ to be the space of images.
Each perturbation $a \in \cA$ is associated with some ground truth distribution $P^a \in \cP$ over outcomes, where $\cP$ denotes the set of probability distributions on $\cX$.
In practice, we may only observe outcome data for a small subset $\Aobs \subseteq \cA$ of the overall perturbation space.
For example, when experiments are costly, as in perturbational cell imaging, we may only have observations for the control condition and for individual drugs at fixed dosages (i.e., $\Aobs = \{ \bzero \} \cup \{ \be_k : k \in [K] \}$, where $\be_k$ is the $k^{\text{th}}$ basis vector).

In such cases, we would like to \emph{extrapolate} to an unseen perturbation $a^* \in \cA \setminus \Aobs$.
Under a variety of well-motivated theoretical assumptions, the target distribution $P^{a^*}$ can often be identified from the observed distributions $( P^a )_{a \in \Aobs}$.
In this work, we consider a form that is common in several settings, where the density of distribution $P^{a^*}$ is expressed as a geometric mixture of the source distribution densities.
To ensure that our expression is well-defined, we assume throughout the paper that, for each $a \in \cA$, the distribution $P^a$ is associated with a density over $\cX$.
With a typical abuse of notation, we also write this density as $P^a$.
Further, we assume that these densities are positive, i.e., for all $a \in \cA_o$ and $x \in \cX$, we have $P^a(x) > 0$, ensuring that we avoid divisions by zero.\footnote{In \Cref{appendix:formal-assumption}, we give a more formal version of this assumption, and discuss the more formal interpretation of our equality expressions.}

\begin{definition}\label{def:weighted-composition}
    Fix $\Aobs \subseteq \cA$.
    Given a collection of distributions $\bP = (P^a)_{a \in \Aobs}$ and a function $w\colon \Aobs \to \bbR$, we say that $w$ is a \defword{valid weighting} if $f(x) = \prod_{a \in \Aobs} P^a(x)^{w_a}$ is integrable, where $w_a$ is shorthand for $w(a)$.
    For a valid weighting $w$, we define the \defword{weighted composition} of $\bP$ as 
    \begin{equation}\label{eqn:weighted-composition}
        \Comp_w(\bP)(x) \defeq \frac{1}{Z^w} \prod\nolimits_{a \in \Aobs} P^a(x)^{w_a},
        \quad
        \text{where}\
        Z^w \defeq \int_\cX \left( \prod\nolimits_{a \in \Aobs} P^a(x)^{w_a} \right) \d x.
    \end{equation}
    Alternatively, we write $\bP^w \defeq \Comp_w(\bP)$ as shorthand.
\end{definition}

This geometric combination at the distribution level becomes a linear combination at the score level.
Assuming $\log P^a$ is differentiable, the \emph{Stein score functions} are $S^a(x) \defeq \nabla_\xsym \log P^a(x)$ for $a \in \Aobs$, and $S^w(x) \defeq \nabla_\xsym \log P^w(x)$.
Then, for any valid weighting $w$, we have $S^w(x) = \sum\nolimits_{a \in \Aobs} w_a S^a(x)$.

\paragraph{Causal origins of weighted compositions}
To better understand the importance of weighted compositions, it is useful to briefly discuss the conditions under which they naturally arise.
As a canonical example, consider predicting the outcome of a double perturbation $a^* = \be_1 + \be_2$, given only observational data and single perturbations, i.e., $\Aobs = \{ \bzero, \be_1, \be_2 \}$, and suppose the outcome space $\cX$ can be decomposed into an upstream component $\cX_1$ and a downstream component $\cX_2$.
Reflecting this structure, we can factorize $\bP^\bzero$ as $P^\bzero(x) = P^\bzero(x_1) P^\bzero(x_2 \mid x_1)$.
In many systems, perturbations can be expected to have isolated, targeted effects: in causality, this can be expressed in terms of \emph{mechanism changes} or \emph{soft interventions} \citep{squires2023causal}.
In particular, suppose $\be_1$ affects the upstream component but not the downstream one, i.e., $P^{\be_1}(x) = P^{\be_1}(x_1) P^{\bzero}(x_2 \mid x_1)$, while $\be_2$ affects the downstream component but not the upstream one, i.e., $P^{\be_2}(x) = P^{\bzero}(x_1) P^{\be_2}(x_2 \mid x_1)$.

Further, perturbations may often be expected to have modular effects, known as the principle of \emph{independent causal mechanisms} \citep{scholkopf2021toward}.
Under this assumption, the double perturbation $a^*$ simply ``combines'' the effects of the single perturbations, so that $P^{a^*}(x) = P^{\be_1}(x_1) P^{\be_2}(x_2 \mid x_1)$.
In this case, it is simple to show that $P^{a^*} = \Comp_{w_\double}(\bP)$ for the valid weighting $w(\be_1) = w(\be_2) = 1$ and $w(\bzero) = -1$.
More broadly, this reasoning can be extended to more than two perturbations, and $\cX$ does not have to be directly decomposable into causal components: this decomposition may only hold in some unknown latent space.
In such cases, the field of \emph{causal representation learning} provides a more general justification for weighted compositions, as we describe in \Cref{sec:perspectives}.


\paragraph{The compositional generation task and its challenges}
With these motivating examples in mind, we define a concrete version of the compositional generation task as follows:

\noindent
\begin{tcolorbox}
\begin{task}[Compositional generation for weighted compositions]
    As inputs, take a conditional diffusion model $\bP_\theta \approx \bP_\star$ from $\Aobs$ to $\cX$ and a valid weighting $w\colon \Aobs \to \bbR$ for $\bP_\star$.
    Using an efficient algorithm, produce samples from a distribution $\tP \in \cP$ such that $\tP \approx \bP_\star^w$.

    More formally, let $d$ be some metric on the space of distributions $\cP$.
    Then, our input assumption ($\bP_\theta \approx \bP_\star$) says that we are given $\bP_\theta$ such that $d(P^a_\theta, P^a_\star) < \epsilon_1$ for all $a \in \cA_o$, and our output requirement ($\tP \approx \bP_\star^w$) says that we produce samples from some $\tP$ such that $d(\tP, \bP^w_\star) < \eps_2$.
\end{task}
\end{tcolorbox}

We remark that this task encounters two challenges.
First, one encounters \emph{inference-time approximation error}: for most model classes (e.g. diffusion models), we cannot efficiently sample from $\bP_\theta^w$, only some proxy $\tP \approx \bP^w_\theta$.
Second, we encounter \emph{score estimation error}: since we are using an estimated model $\bP_\theta \neq \bP_\star$, we generally have $\bP_\theta^w \neq \bP_\star^w$ for $\bP_\theta^w \defeq \Comp_w(\bP_\theta)$.
In particular, given any metric $d$ on the set of probability distributions $\cP$, the triangle inequality gives
\begin{equation*}
    d(\tP, \bP_\star^w) \leq 
    \underbrace{d(\tP, \bP_\theta^w)}_{\text{inference-time approximation error}}
    +
    \underbrace{d(\bP_\theta^w, \bP_\star^w)}_{\text{score estimation error}}.
\end{equation*}
As we discuss in \Cref{sec:composition}, recent methods such as \emph{Feynman-Kac correction} \citep{skreta2025feynman} provide elegant solutions to reduce the inference-time approximation error $d(\tP, \bP_\theta^w)$ when composing diffusion models.
Although reducing this error is important, we find that the score estimation error $d(\bP_\theta^w, \bP_\star^w)$, often overlooked in previous works, can be much more important in certain settings.
Specifically, we consider settings where typical samples from the weighted composition $\bP^w$ lie in low-density regions of the source distributions.
In these cases, we say that $\bP^w$ is \emph{out-of-distribution (OOD)}: since the estimated score functions were trained on $\bP = (P^a)_{a \in \cA_o}$, but are being evaluated on $\bP^w$, these cases lead to similar issues as those considered in the field of \emph{OOD generalization} \citep{david2010impossibility,kpotufe2018marginal,canatar2021out}.


 


\section{Technical background}\label{sec:background}
In this section, we review relevant background on conditional generative models and compositional generation.
Here, our focus is on conditional diffusion models, though we expect similar results to apply to other forms of conditional generative models.


\subsection{Conditional diffusion models}
\paragraph{Notation}
We let $\W_t$ denote a standard Wiener process, using the convention that $t=0$ is data and $t=1$ is noise.
We use the terms ``noising'' and ``denoising'' rather than ``forward'' and ``reverse'' to avoid confusion arising from conflicting conventions between diffusion and flow-based models. All denoising SDEs and ODEs are written in the same time variable $t$, but integrated from $t=1$ to $t=0$, and the SDE uses the reverse-time Wiener process $\rW_t$.



\paragraph{Noising}

Given a drift coefficient schedule $( \mu_t )_t$ and a diffusion coefficient schedule $( \sigma_t )_t$, we assume a \defword{noising SDE} of the form $\d x_t = u_t(x_t) \d t + \sigma_t \d W_t$, with $u_t(x) \defeq \mu_t x$,
which we refer to as the \emph{noising drift function}.
Given $x_0 \sim P^a$, the solution to this SDE is a stochastic process $(\rvX^a_t)_{t}$, and defines the \defword{noising transition kernel} $\Noise_t\colon \cX \to \cP(\cX)$, where
\begin{equation}\label{eq:noise}
    \Noise_t(x_0)
    =
    \cN(\alpha_t x_0 ; \gamma_t^2 I),\quad
    \text{for}\quad
    \alpha_t \defeq \exp \left( \int_0^t \mu_s \d s \right)\ \text{and}\
    \gamma_t \defeq \alpha_t \sqrt{\int_0^t \frac{\sigma_s^2}{\alpha_s^2} \d s}.
\end{equation}
This kernel extends naturally to a function on distributions.
With some abuse of notation, we write $\Noise_t\colon \cP(\cX) \to \cP(\cX)$, where $\Noise_t\colon P \mapsto \int \Noise_t(x) P(x) \d x$, and we write $P_t^a \defeq \Noise_t(P^a)$ for the noised distribution.
Put differently, $P_t^a$ is the marginal distribution of samples at time $t$ after drawing an unnoised sample from $P^a$.
Then, each noised distribution $P^a_t$ has the associated Stein score function $S^a_t(x) \defeq \nabla_\xsym \log P^a_t(x)$. 
For any noising SDE, there exists a corresponding deterministic ODE with identical marginal distributions $P^a_t$ at all times $t\in[0,1]$ \citep{song2021scorebased}.
Hence, we define this \defword{noising ODE} as $\d x_t = \tilde{u}_t(x_t) \d t$ for $\tilde{u}_t(x) \defeq u_t(x) - \frac{1}{2}\sigma_t^2 S^a_t(x)$.

\paragraph{Denoising}
Diffusion models leverage the fact that the stochastic process $(\bX^a_t)_t$ also solves the \defword{denoising SDE}
\begin{equation}\label{eqn:denoising-sde}
    \d x^a_t = v^a_t(x^a_t) \d t + \sigma_t \d \rW_t, 
    \quad\text{for}\quad
    v^a_t(x) \defeq u_t(x) - \sigma_t^2 S^a_t(x),
\end{equation}
as given in \citet{song2021scorebased}.
Alternatively, the \defword{denoising ODE} is defined by running the noising ODE backward in time.
In practice, the conditional score functions $S^a_t$ can be estimated by a variety of methods \citep{hyvarinen2005score, vincent2011connection, denoising-diffusion-probabilistic-models, song2021scorebased, lipman2023flow, albergo2023stochastic}, with the conditional score functions typically parameterized as conditional deep networks \citep{perez2018, ho2022classifier}. 
We consider ``vanilla'' models in the sense that they are trained only to estimate scores (or some known transformations thereof) on observed conditions, using standard objectives and shared parameters, without architectural or training-time constraints tailored to compositional extrapolation. Our experiments use denoising diffusion \citep{denoising-diffusion-probabilistic-models}, but the results apply to other vanilla score estimators (up to differences in estimation error).

\subsection{Composition methods}\label{sec:composition}
Now, we discuss methods for sampling from weighted compositions of conditional diffusion models.
Noising $\bP^w \defeq \Comp_w(\bP)$ and taking its score, we obtain the \defword{correctly-composed score function}
\begin{equation*}
    S^w_t(x) \defeq \nabla_\xsym \log \bP^w_t(x),
    \quad
    \text{for}\
    \bP^w_t \defeq \Noise_t(\bP^w)
\end{equation*}
Ideally, we would like to run the denoising SDE in \Cref{eqn:denoising-sde} with $v^w_t(x) \defeq u_t(x) -\sigma_t^2 S^w_t(x)$ in place of $v^a_t(x)$.
However, we do not know $S^w_t(x)$ or have a simple way to compute it. Instead, a simple heuristic used in many prior works (e.g., classifier-free guidance \citep{ho2022classifier} and compositional visual generation \citep{liu2022compositional}) substitutes this score with an additive approximation.
In particular, the heuristic uses
the \defword{na{\"i}vely-composed score function}
\begin{equation*}
    S^{w,\naive}_t(x)
    \defeq
    \sum\nolimits_{a \in \Aobs} w_a S^a_t(x).
\end{equation*}
The corresponding \defword{na\"ive denoising SDE} is thus given by
\begin{equation}\label{eqn:naive-denoising-sde}
    \d x_t^{w,\nvs}
    =
    v^{w,\naive}_t(x_t^{w,\nvs}) \d t + \sigma_t \d \rW_t,
    \quad\text{where}\quad
    v^{w,\naive}_t(x)
    \defeq
    u_t(x) - \sigma^2_t S_t^{w,\naive}(x),
\end{equation}
and the corresponding \defword{na\"ive denoising ODE} is
\begin{equation}\label{eqn:naive-denoising-ode}
    \d x^{w,\nvo}_t = \tilde{v}^{w,\naive}_t(x^{w,\nvo}_t) \d t \quad\text{where}\quad \tilde{v}^{w,\naive}_t(x) \defeq u_t(x) - \frac{1}{2}\sigma_t^2 S_t^{w,\naive}(x).
\end{equation}

Notably, the definition of weighted composition ensures that $S_t^{w,\naive}(x) = S_t^w(x)$ at $t=0$, and the definition of the noising SDE ensures that, if $\sum_a w_a=1$, then $S_t^{w,\naive}(x) = S_t^w(x)$ at $t=1$.
However, it is not generically true that $S_t^{w,\naive}(x) = S_t^w(x)$ for intermediate times $t\in(0,1)$.
Due to this discrepancy, na{\"i}ve denoising methods can introduce substantial inference-time approximation error, as we illustrate in \Cref{sec:multivariate-gaussian}.



In contrast to na\"ive methods, Feynman-Kac Correctors \citep{skreta2025feynman} explicitly account for this discrepancy through a particle-based Sequential Monte Carlo approach. By maintaining $K$ weighted particles updated according to a residual term $g_t(x)$, and resampling according to these weights, FKC tracks the distribution of a target along the entire denoising trajectory. We refer readers to \Cref{appendix:fkc} for further theoretical intuition on FKC.

\newtheorem*{remark}{Remark}
\section{Theoretical insights}\label{sec:theoretical-insights}


To obtain a more fine-grained view of approximation error, we begin by studying weighted compositions of multivariate Gaussians, which also form the basis for our first experiments in \Cref{sec:experiments}. 

\subsection{Multivariate Gaussians}\label{sec:multivariate-gaussian}

\begin{theorem}
    \label{theorem:gaussian-integrability}
    Let $P^a = \cN(0, \Sigma_a)$ for $a \in \Aobs$.
    Then, $w$ is a valid weighting function if and only if $\sum_{a \in \Aobs} w_a \Sigma_a^{-1} \succ 0$.
    In this case, $\bP^w(x) = \cN(0, \Sigma^w)$ for 
    \begin{equation}\label{eq:gaussian-precision}
        \Sigma^w = \left( \sum\nolimits_{a \in \Aobs} w_a \Sigma_a^{-1} \right)^{-1}
    \end{equation}
\end{theorem}

The proof is given in \Cref{appendix:proof-gaussian-integrability}.
However, even for simple and analytically tractable Gaussians, the composition operator does not commute with noising, which leads the na\"ive denoising process to sample from the incorrect distribution:

\begin{theorem}\label{thm:commutative-gaussian-ode}
    Let $P^a = \cN(0, \Sigma_a)$ for $a \in \Aobs$, and assume $\{\Sigma_a\}_{a \in \Aobs}$ are mutually commutative and that $\sum_{a \in \Aobs} w_a = 1$. Then, the na\"ive denoising ODE \eqref{eqn:naive-denoising-ode} yields $\tP^{w,\nvo} = \cN(0, \Sigma^{w,\nvo})$, where
\begin{equation*}
    \Sigma^{w,\nvo} = \prod\nolimits_{a \in \Aobs} \Sigma_a^{w_a}.
\end{equation*}
\end{theorem}
The proof is given in \Cref{appendix:proof-commutative-gaussian-ode}. 
In this case, as a result of this approximation error, addition and subtraction are effectively replaced by multiplication and division after the denoising process. These resulting quantities can in fact be \emph{arbitrarily far from each other}, as a simple example demonstrates: 
take $|\Aobs| = 2$, $w_1 = w_2 = \tfrac{1}{2}$, and $\Sigma_1 = \varepsilon I$, $\Sigma_2 = \varepsilon^{-1} I$ for $\varepsilon > 0$. Then $\Sigma^{w,\nvo} = I$ and $\Sigma^w= \frac{2\varepsilon}{1 + \varepsilon^2}\, I$,
so $\|\Sigma^{w,\nvo}\|_{\mathrm{op}} / \|\Sigma^w\|_{\mathrm{op}} 
= (1 + \varepsilon^2)/(2\varepsilon) \to \infty$ as $\varepsilon \to 0^+$.
Accordingly, this motivates the need for methods of reducing inference-time approximation error in weighted compositions. 
However, as we will see, there are also distributions of interest for which approximation error is \textit{not} a concern.

\subsection{Base-composed distributions}


As a special case of weighted compositions, we are interested in studying what we refer to as \defword{base-composed distributions}: $\bP^{w}$ for the valid weighting $w_a = 1 -\kron_{a=0}\cdot|\Aobs| ~~\forall a \in \Aobs$, with the corresponding distribution given by
\begin{equation}\label{eq:base-composed-distribution}
   \bP^w(x)= \frac{1}{Z^w} P^0(x)\prod\nolimits_{a \in \Aobs} \frac{P^a(x)}{P^0(x)}.
\end{equation}

Recall that, in causal representation learning, these compositions are motivated by a view of perturbations as interventions.
As an example, the base case $P^0$ may represent a control trial while the other perturbation distributions model the effects of specific drugs.






\subsubsection{Factorized conditionals} 
In addition to being theoretically well-motivated, some base-compositions admit especially helpful properties. 
Of these, we focus on \emph{Factorized Conditionals}, which can be interpreted as a causal model over disconnected components (see \Cref{sec:perspectives}).
\begin{definition}[\cite{bradley2025mechanisms}, modified]
    A collection of distributions $(P^0, P^{a_1}, \allowbreak P^{a_2}, \allowbreak \dots \allowbreak P^{a_k})$ over $\mathbb{R}^n$ are \defword{Factorized Conditionals} if there exists a partition $M_0, M_1, \dots M_k$ of $[n]$ such that
    \begin{equation}\label{eq:factorized-conditionals}
        P^{a_i}(x) = P^{a_i}(x_{M_{i}})P^0(x_{M_{i}^c}) \qquad \text{and} \qquad P^0(x) = P^0(x_{M_{0}}) \prod\nolimits_{i\in[k]} P^0(x_{M_{i}}),
    \end{equation}
    where $M_i^c$ denotes the set complement of $M_i$.
\end{definition}
Visually, this could correspond to placing objects which occur in disjoint regions of an image (e.g., couch and wall-mounted painting) against a shared background scene $P^0$ (e.g., empty room).
If a set of distributions obey these properties, then \textit{assuming scores with zero error}, running the na{\"i}ve denoising SDE on the base-composed distribution correctly samples from \eqref{eq:base-composed-distribution} at $t=0$.
More generally, if \eqref{eq:factorized-conditionals} holds in a latent feature space related to the input by an invertible orthogonal transform, the naive denoising SDE still correctly samples the base composition \eqref{eq:base-composed-distribution} \citep{bradley2025mechanisms}.
For instance, conditions could orthogonally represent the style (e.g. ``realistic'' or ``impressionist'') of an image versus its content (e.g. ``cat'' or ``dog'').

Furthermore, if a set of distributions are Factorized Conditionals, then as we prove in \Cref{appendix:g-is-zero-in-FC}, we can also make strong claims about how FKC behaves on their base composition:

\begin{theorem}\label{thm:fkc-factorized-conditional}
    Assume the set of distributions $(P^0, P^{a_1}, P^{a_2}, \dots P^{a_k})$ are Factorized Conditional. Then, for the base-composition on $(P^0, P^{a_1}, P^{a_2}, \dots P^{a_k})$, the Feynman-Kac Corrector log-weight drift vanishes:
    \begin{equation*}
        g_t(x) = 0 \;~ \text{for all} ~\; x \in \mathbb{R}^n, ~t \in [0,1],
    \end{equation*}
    and the Feynman-Kac Corrector denoising process coincides exactly with the na\"ive denoising SDE.
\end{theorem}


\begin{remark}
    Under these conditions, FKC provides no benefit as there is no inference-time approximation error to correct. Further, with learned scores $S^a_\theta = S^a + \varepsilon^a$ (due to estimation error), the 
    cancellation which caused the residual $g_t$ to become $0$ no longer holds exactly. As a result, $g_t$ is driven entirely by estimation noise, so resampling concentrates particles on noise rather than signal. Consequently, FKC drifts further from the true distribution rather than closer, and this effect worsens for out-of-distribution compositions, where $\varepsilon^a$ is expected to be larger.
\end{remark}




\section{Experiments}\label{sec:experiments}
The goal of our empirical evaluation is to disentangle the two highlighted sources of error in compositional generation: inference-time approximation error and score estimation error.
We partition our tasks into In-Distribution (ID) settings, where the composed target $\bP^w$ is well-represented within the training data sampled from the source distributions, and Out-of-Distribution (OOD) settings, where the composed target does not lie within the effective training support of any of the distributions. 
This affords us principled control over the score estimation error.
To separate out the effects of approximation error, we further adjust distributions to either be Factorized Conditionals or Non-Factorized Conditionals.

In each experiment, we define two perturbation distributions $P^{a_1}$ and $P^{a_2}$ alongside a base distribution $P^0$. 
A conditional diffusion model $\bP_\theta$ is trained to sample from these distributions (training details are provided in Appendix~\ref{appendix:exp}), after which FKC is applied to this model to sample from the base-composed distribution \eqref{eq:base-composed-distribution}. 
We control the strength of the Feynman-Kac correction by altering the number of simulated particles ($K$), and we exert further control over the score estimation error by varying the number of training samples ($N$) or by using the analytic score functions.
On synthetic experiments, we evaluate using the sliced Wasserstein-2 distance ($\mathrm{SW}_2$) and Maximum Mean Discrepancy ($\mathrm{MMD}^2$) compared to the ground truth base-composition (\Cref{appendix:exp}). 
In the main text, we report $\mathrm{SW}_2$ using mean values, while \Cref{appendix:tables} provides both metrics and reports standard deviation.

In addition to conditional models, we also evaluate compositions of ``expert'' models each trained only on a single condition. 
Notably, we find that score estimation error worsens in these settings, suggesting that the weight-sharing of the conditional diffusion model plays an important role in implicitly regularizing the models towards compositional accuracy (\Cref{appendix:sep-vs-conditional}).
The code for the experiments is provided at \url{https://github.com/DSoiffer/compositional-diffusion}.

\begin{figure}[htbp]
    \centering
    \begin{subfigure}[t]{0.41\linewidth}
        \centering
        \includegraphics[width=\linewidth]{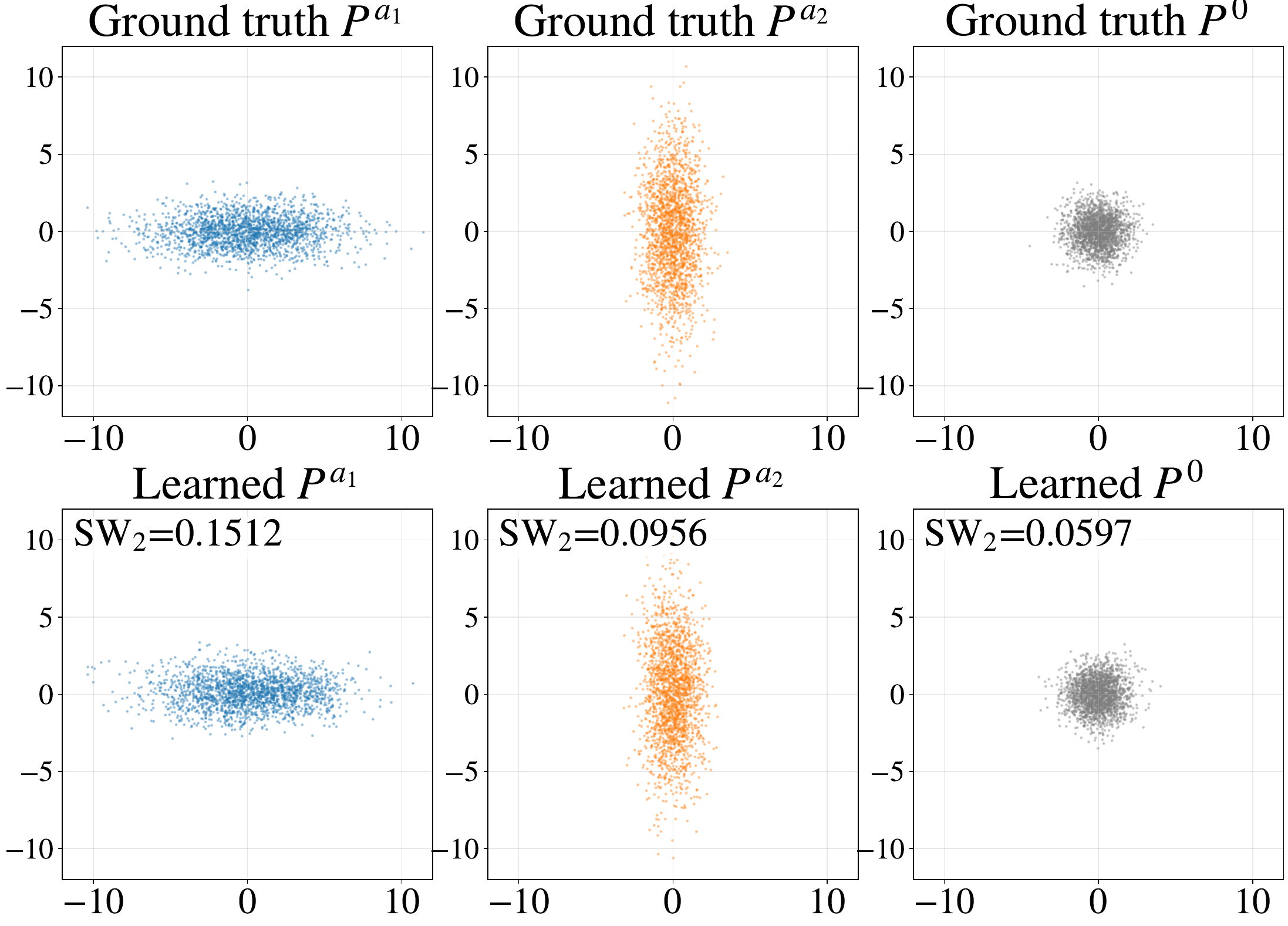}
        \caption{Individual distributions}
        \label{fig:analytical}
    \end{subfigure}
    \hfill
    \begin{subfigure}[t]{0.58\linewidth}
        \centering
        \includegraphics[width=\linewidth]{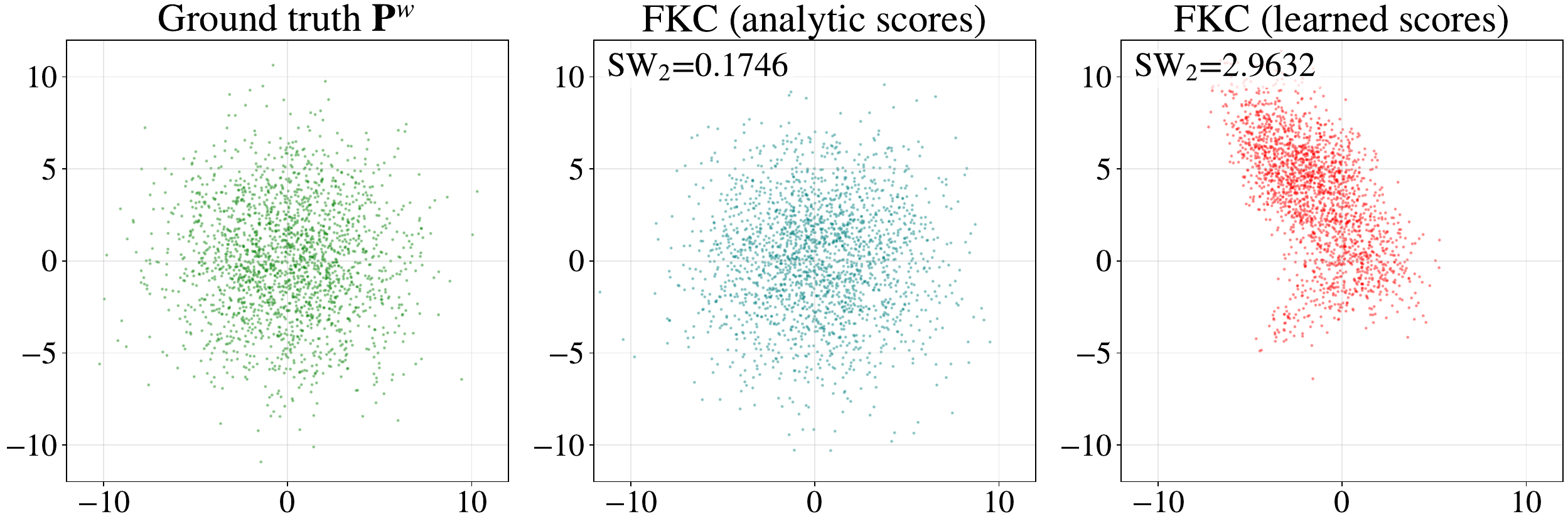}
        \caption{Composition}
        \label{fig:learned}
    \end{subfigure}
    \caption{\textbf{Feynman-Kac Correctors accurately estimate the composition with analytical scores, but fail with learned distributions due to out-of-distribution estimation error.} Learned and analytical distributions for $P^0 =\mathcal{N}\left(\mathbf{0}, \begin{bsmallmatrix} 1 & 0\\ 0 & 1 \end{bsmallmatrix}\right)$, $P^{a_1}=\mathcal{N}\left(\mathbf{0}, \begin{bsmallmatrix} 10 & 0\\ 0 & 1 \end{bsmallmatrix}\right)$ and $P^{a_2}=\mathcal{N}\left(\mathbf{0}, \begin{bsmallmatrix} 1 & 0\\ 0 & 10 \end{bsmallmatrix}\right)$ and the composition $P(x) \propto \frac{P^{a_1}(x)P^{a_2}(x)}{P^0(x)}$. Conditional diffusion models accurately learn the conditional distributions within high probability regions of the training data, but fail to compose correctly out-of-distribution even with FKC ($K=16$ particles). } 
    \label{fig:2d-simple-gaussian-experiment}
\end{figure}

\subsection{Synthetic data: 2D Gaussian}\label{subsec:2d-gaussian}

\begin{table*}[t]
    \centering
    \caption{
        \textbf{The relative magnitudes of inference-time approximation error and score estimation error are highly dependent on problem setting.}
        FKC is applied with $K$ particles to a conditional diffusion model trained on $N$ samples, or to the analytic score functions.
        Each value is the $\mathrm{SW}_2$ between 5000 samples from the target distribution $\bP^w$ and $5000$ empirical samples, averaged over 30 training runs.
    }
    \label{tab:2d-combined_results}

    \begin{subtable}[t]{0.48\textwidth}
        \centering
        \caption{\textbf{Factorized conditionals + ID}}
        \begin{tabular}{@{}l|cccc@{}}
        \toprule
        $K \backslash N$ & 100 & 1000 & 10000 & analytic \\
        \midrule
        $1$ & $0.3537$ & $0.1089$ & $0.0747$ & $0.0354$ \\
        $4$ & $0.4276$ & $0.0974$ & $0.0750$ & $0.0376$ \\
        $16$ & $0.4894$ & $0.0956$ & $0.0772$ & $0.0351$ \\
        $64$ & $0.5136$ & $0.0937$ & $0.0758$ & $0.0356$ \\
        $256$ & $0.5261$ & $0.1010$ & $0.0777$ & $0.0346$ \\
        \bottomrule
        \end{tabular}
    \end{subtable}\hfill
    \begin{subtable}[t]{0.48\textwidth}
        \centering
        \caption{\textbf{Non-factorized conditionals + ID}}
        \begin{tabular}{@{}l|cccc@{}}
        \toprule
        $K \backslash N$ & 100 & 1000 & 10000 & analytic \\
        \midrule
        $1$ & $0.4454$ & $0.1234$ & $0.0957$ & $0.0655$ \\
        $4$ & $0.5331$ & $0.0969$ & $0.0751$ & $0.0412$ \\
        $16$ & $0.5914$ & $0.0901$ & $0.0750$ & $0.0353$ \\
        $64$ & $0.6111$ & $0.0905$ & $0.0740$ & $0.0357$ \\
        $256$ & $0.6216$ & $0.0963$ & $0.0793$ & $0.0333$ \\
        \bottomrule
        \end{tabular}
    \end{subtable}

    \vspace{1em} 
    
    \begin{subtable}[t]{0.48\textwidth}
        \centering
        \caption{\textbf{Factorized conditionals + OOD}}
        \begin{tabular}{@{}l|cccc@{}}
        \toprule
        $K \backslash N$ & 100 & 1000 & 10000 & analytic \\
        \midrule
        $1$ & $0.9658$ & $0.5577$ & $0.3948$ & $0.1122$ \\
        $4$ & $2.0688$ & $1.6885$ & $1.1773$ & $0.1130$ \\
        $16$ & $2.9177$ & $2.7214$ & $1.9699$ & $0.1133$ \\
        $64$ & $3.3210$ & $3.3348$ & $2.4983$ & $0.1124$ \\
        $256$ & $3.5097$ & $3.6426$ & $2.8210$ & $0.1140$ \\
        \bottomrule
        \end{tabular}
    \end{subtable}\hfill
    \begin{subtable}[t]{0.48\textwidth}
        \centering
        \caption{\textbf{Non-factorized conditionals + OOD}}
        \begin{tabular}{@{}l|cccc@{}}
        \toprule
        $K \backslash N$ & 100 & 1000 & 10000 & analytic \\
        \midrule
        $1$ & $1.0957$ & $0.9154$ & $0.8518$ & $0.6375$ \\
        $4$ & $2.1611$ & $1.8306$ & $1.3763$ & $0.1764$ \\
        $16$ & $3.0011$ & $2.8830$ & $2.0117$ & $0.0981$ \\
        $64$ & $3.4224$ & $3.6247$ & $2.5011$ & $0.0831$ \\
        $256$ & $3.6076$ & $3.9971$ & $2.8322$ & $0.0820$ \\
        \bottomrule
        \end{tabular}
    \end{subtable}
    
\end{table*}


We begin with a two-dimensional Gaussian toy setting, where all key quantities have a closed form.
Throughout, we take $P^{a_1}=\mathcal{N}\left(\mathbf{0}, \begin{bsmallmatrix} 10 & 0\\ 0 & 1 \end{bsmallmatrix}\right)$ and $P^{a_2}=\mathcal{N}\left(\mathbf{0}, \begin{bsmallmatrix} 1 & 0\\ 0 & 10 \end{bsmallmatrix}\right)$.
The base distribution $P^0$ varies across four settings, described below. Results are presented in \Cref{tab:2d-combined_results}. 


\textbf{Factorized conditionals.} For the in-distribution setting, we set $P^0=\mathcal{N}\left(\mathbf{0}, \begin{bsmallmatrix} 10 & 0\\ 0 & 10 \end{bsmallmatrix}\right)$, yielding (by \eqref{eq:gaussian-precision}) the target composition $\bP^w = \mathcal{N}\left(\mathbf{0}, \begin{bsmallmatrix} 1 & 0\\ 0 & 1 \end{bsmallmatrix}\right)$. Because this composed distribution lies  within the effective training support of $P^0$, $P^{a_1}$, and $P^{a_2}$, na\"ive denoising is already highly accurate. Consequently, applying FKC yields no improvements, and actively degrades performance for undertrained models. 
For the OOD configuration, we set $P^0 =\mathcal{N}\left(\mathbf{0}, \begin{bsmallmatrix} 1 & 0\\ 0 & 1 \end{bsmallmatrix}\right)$, yielding $\bP^w = \mathcal{N}\left(\mathbf{0}, \begin{bsmallmatrix} 10 & 0\\ 0 & 10 \end{bsmallmatrix}\right)$. 
Although na{\"i}ve denoising achieves a reasonable fit, out-of-distribution estimation error remains 
problematic. When FKC is applied, this error compounds drastically, quickly diverging away from the correct distribution as particle count increases, as shown in \Cref{fig:2d-simple-gaussian-experiment}.

\textbf{Non-factorized conditionals.} For the in-distribution setting, we set $P^0 = \mathcal{N}\left(\mathbf{0}, \begin{bsmallmatrix} 20 & 0\\ 0 & 20 \end{bsmallmatrix}\right)$, yielding $\bP^w = \mathcal{N}\left(\mathbf{0}, \frac{20}{21}I\right)$. 
Here, na{\"i}ve denoising consistently yields poor distributional fit, in line with expectations over approximation error. 
However, FKC successfully leverages well-trained models to correct this approximation error as the particle count increases. 
In the OOD setting, we set $P^0 = \mathcal{N}\left(\mathbf{0}, \begin{bsmallmatrix} 1.1 & 0\\ 0 & 1.1 \end{bsmallmatrix}\right)$, so $\bP^w = \mathcal{N}\left(\mathbf{0}, \frac{110}{21}I\right)$. In this regime, distributional fit remains poor across all non-analytical settings, and FKC strictly worsens performance when applied to learned models. It is only when the algorithm is supplied with the exact analytical scores that it achieves the expected theoretical improvements.

\subsection{Synthetic data: mixture of Gaussians}\label{subsec:gmm}

\begin{table*}[t]
    \centering
    \caption{
        \textbf{The expected trends also hold for Gaussian mixture models.}
        The table follows the same conventions as \Cref{tab:2d-combined_results}.
    }
    \label{tab:score_sw2_combined}
    
    \begin{subtable}[t]{0.48\textwidth}
        \centering
        \caption{\textbf{In-Distribution}}
        \begin{tabular}{@{}l|cccc@{}}
\toprule
$K \backslash N$ & 100 & 1000 & 10000 & analytic \\
\midrule
$1$ & $1.2113$ & $1.2470$ & $1.1704$ & $1.2133$ \\
$4$ & $0.6678$ & $0.2842$ & $0.2318$ & $0.2414$ \\
$16$ & $0.7735$ & $0.2206$ & $0.1454$ & $0.1044$ \\
$64$ & $0.8130$ & $0.2134$ & $0.1478$ & $0.0736$ \\
$256$ & $0.8288$ & $0.1985$ & $0.1350$ & $0.0771$ \\
\bottomrule
\end{tabular}
    \end{subtable}\hfill
    \begin{subtable}[t]{0.48\textwidth}
        \centering
        \caption{\textbf{Out-of-Distribution}}
        \begin{tabular}{@{}l|cccc@{}}
        \toprule
        $K \backslash N$ & 100 & 1000 & 10000 & analytic \\
        \midrule
$1$ & $4.0336$ & $4.0508$ & $3.9606$ & $4.2102$ \\
$4$ & $2.7357$ & $2.7265$ & $2.4675$ & $2.5081$ \\
$16$ & $2.4276$ & $2.4556$ & $2.1718$ & $1.7704$ \\
$64$ & $2.6212$ & $2.7205$ & $2.4005$ & $1.5531$ \\
$256$ & $3.0238$ & $3.0562$ & $2.6706$ & $1.4670$ \\
        \bottomrule
        \end{tabular}
    \end{subtable}
    
\end{table*}



In a slightly more complex synthetic setting, we consider two GMM experiments that both violate the factorized conditionals assumption. In both experiments, $\mathcal{X} = \mathbb{R}^2$ and we define $P^0$, $P^{a_1}$, and $P^{a_2}$ as mixtures of Gaussians, where each component has covariance matrix $\sigma I$.
Because $P^w$ is a ratio of Gaussian mixtures,
its normalizing constant is analytically intractable and the
resulting density is not itself a GMM.
We therefore obtain ground-truth samples via rejection sampling for in-distribution setting and importance sampling for out-of-distribution setting (\ref{appendix:gmm-exp}). Results are provided in Table~\ref{tab:score_sw2_combined}.


For the in-distribution setting, component means are arranged on a $2 \times 3$ grid,
$\mu_k \in \{-3, 0, 3\} \times \{-2.5, 2.5\}$, and we set $\sigma = 0.6$.
The base $P^0$ assigns uniform weight to all six components.
Conditions restrict to overlapping subsets:
$P^{a_1}$ is uniform over $\{k_0, k_1, k_2, k_3\}$,
and $P^{a_2}$ is uniform over $\{k_1, k_2, k_4, k_5\}$. This configuration is explicitly constructed to isolate approximation error: the unnormalized target density $P^{a_1}(x)P^{a_2}(x) / P^0(x)$ is confined to regions with high probability mass under the conditionals. By ensuring the composed target modes are well-represented in the training data, score estimation error is minimized. In this regime, FKC successfully corrects for the non-factorized nature of the composition, aligning the empirical distributions with the ground truth. Only for poorly learned models ($N=100$) and at large particle counts does further increasing the number of particles begin to harm performance.  

For the out-of-distribution setting, we design the modes so the
target distribution concentrates at locations that are low-probability under all conditionals. 
$P^{a_1}$ is a two-component GMM with means $\{(-6,1.5),(0,0)\}$,
$P^{a_2}$ has means $\{(0,-9),(-1.5,6)\}$,
and $P^0$ has all four means with uniform weight,
all with $\sigma = 1$. While small increases in $K$ yield initial improvements by correcting for approximation error, further scaling $K$ worsens results due to accumulating estimation error; increasing the number of particles is only consistently helpful with oracle scores.

\subsection{Semi-synthetic data: Objects in a room}

\begin{figure}[!ht]
    \centering
    \includegraphics[width=0.99\linewidth]{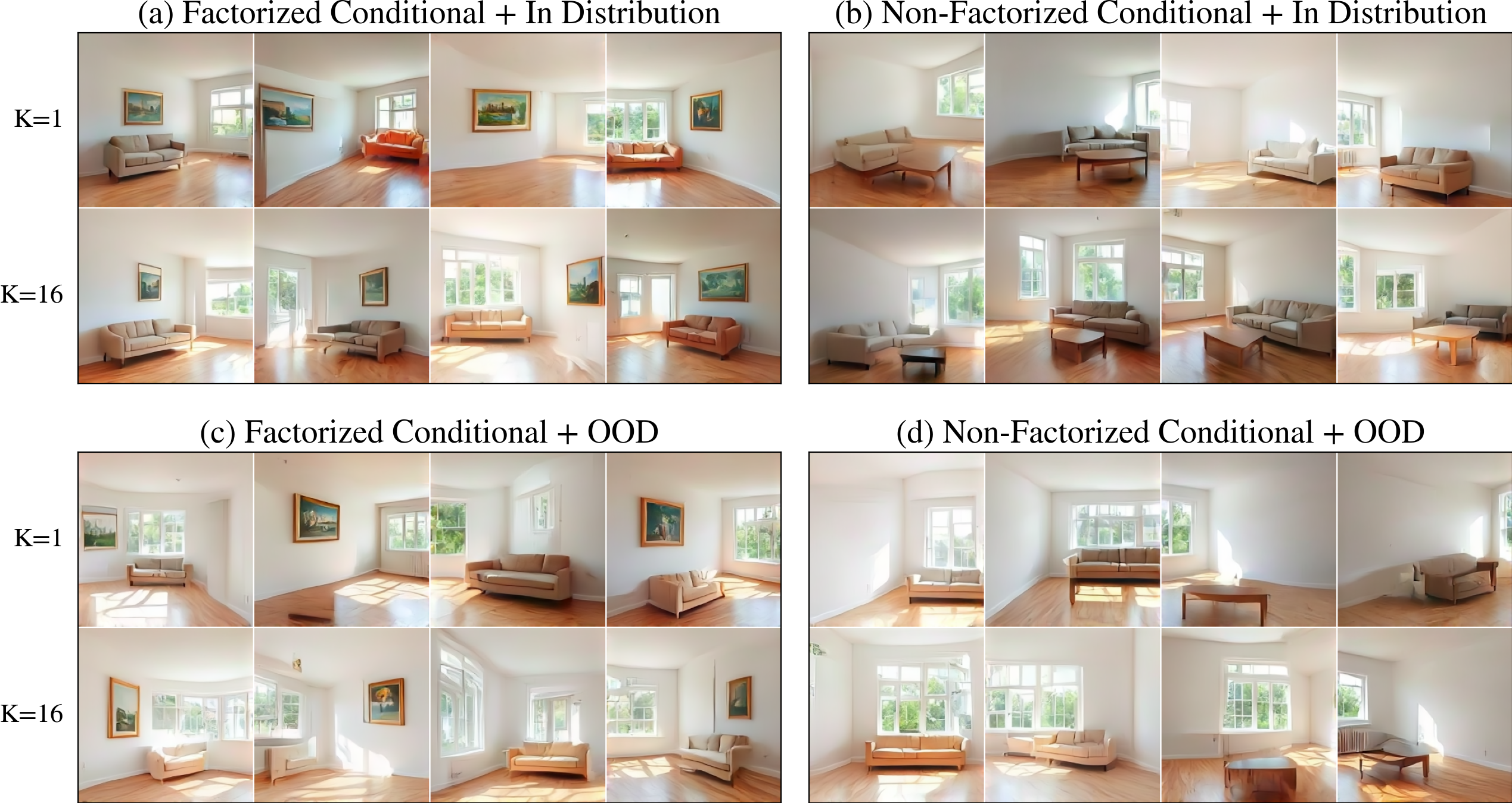}
    \caption{\textbf{Couches and paintings ($\sim$factorized) are largely combined accurately, while couches and tables (non-factorized) are not; FKC is only effective when the target composition is ID.}
    Illustrative samples from base-composed distributions across four experimental settings, comparing na{\"i}ve denoising ($K=1$ particles) against FKC ($K=16$).
    In \textbf{(a)} and \textbf{(c)}, the target distribution is highly concentrated on empty rooms with exactly one couch \textit{and} one painting; in \textbf{(b)} and \textbf{(d)}, it is concentrated on rooms with exactly one couch \textit{and} one coffee table. 
    Underlying conditional mixtures are altered to render the target composition ID or OOD.
    \textbf{(a)} Na\"ive denoising generates couches \textit{and} paintings correctly but occasionally omits one; $K=16$ causes both objects to appear 
    by correcting for mild non-factorization-induced approximation error. \textbf{(b)} Na\"ive frequently fails to generate both a couch \textit{and} table and often warps results; $K=16$ corrects for non-factorization, generating both. \textbf{(c)} Na\"ive samples occasionally contain both objects, but they frequently appear alone; with $K=16$ results are similar but tend to be more warped as OOD estimation error accumulates. \textbf{(d)} The most challenging setting; couches and tables never appear simultaneously and warping is severe.
    }
    \label{fig:room-grid}
\end{figure}

To evaluate these dynamics on more realistic, higher-dimensional visual data, we create a semi-synthetic dataset consisting of realistic room images generated via a text-to-image model. The dataset is partitioned into discrete underlying classes: empty rooms, rooms containing a single object, and rooms containing exactly two objects.
Each base and conditional distribution is a mixture of images from four underlying classes: ``empty,'' ``couch,'' a second object ``$X$,'' and ``couch + $X$.'' We alter object $X$ to either ``framed painting'' or ``coffee table'' to control the factorization of the distributions. When $X$ is a framed painting, the conditionals are approximately factorized, as couches and paintings typically occupy disjoint spatial regions in a room (floor versus wall). Conversely, when $X$ is a coffee table, factorization breaks, as both objects compete for overlapping spatial support on the floor. 

We define the ID and OOD settings by adjusting the probability weights across the four underlying classes, represented as the vector $(P_{\text{empty}}, P_{\text{couch}}, P_X, P_{\text{couch}+X})$. 
In the in-distribution setting, $P^0$ is a uniform mixture $(0.25, 0.25, 0.25, 0.25)$, and the perturbation distributions are heavily biased toward the presence of their respective objects: $P^{a_1}$ is set to $(\alpha, 0.5-\alpha, \alpha, 0.5-\alpha)$ and $P^{a_2}$ is set to $(\alpha, \alpha, 0.5-\alpha, 0.5-\alpha)$, where we use a small smoothing factor $\alpha = 0.01$ to ensure the composition remains well-defined. Under this construction, the composed target distribution places nearly all of its probability mass on the ``couch + $X$'' class, which is well-supported within each condition's training data.
In the OOD setting, the base distribution is defined as $(1.0-3\beta, \beta, \beta, \beta)$ with $\beta = 0.05/3$, heavily skewing it toward the ``empty'' class. $P^{a_1}$ and $P^{a_2}$ are identically skewed toward the isolated ``couch'' and ``$X$'' classes, respectively. As a result, the composed target still concentrates almost entirely on the ``couch + $X$'' class, which is \textit{not} well-supported within each condition's training data.

Illustrative samples from the resulting composed distributions are provided in Figure~\ref{fig:room-grid}. Additional samples are given in \Cref{appendix:room-pics}, and experimental details are given in \Cref{appendix:exp-details-room-objects}.


\section{Discussion}\label{sec:discussion}

In this work, we investigated the limits of compositional generation, presenting a critical negative result for the field: vanilla conditional diffusion models fundamentally struggle to compose when the target distribution is out-of-distribution with respect to the source distributions. 
To conclude, we highlight how this finding should inform future research, along with limitations of our setup.


\textbf{Key takeaways.} 
Crucially, our findings pertain to ``vanilla" diffusion models --- networks trained with standard objectives and shared parameters, utilizing strictly inference-time composition procedures.
Hence, our results suggest that achieving better compositional generation will require interventions earlier in the pipeline, e.g. by training on architectures that learn latent representations, including latent diffusion models \citep{rombach2022high,podell2023sdxl} and, more aptly, models that are explicitly designed for compositional generation, such as object-centric or causal generative models \citep{wu2023slotdiffusion,jiang2023object,komanduri2023identifiable}, or by specialized post-training procedures that encourage the model to extrapolate from high-density regions to low-density ones.

\textbf{Limitations.} We reiterate that our experimental findings only pertain to ``vanilla" (conditional) diffusion models, and are only expected to hold for other ``vanilla" conditional generative models, which requires empirical validation.
In our experiments, we only considered base-composed distributions: a more in-depth study across different weightings may provide deeper insights.
Finally, by focusing only on a toy setting that is exactly analytically solvable, our theoretical results are highly specialized, leaving open fundamental questions about the effect of score estimation error in more general settings.


\section*{Acknowledgements}
This research was developed with funding from the Defense Advanced Research Projects Agency (DARPA) via HR0011-25-3-0239, FA8750-23-2-1015, ONR via N00014-23-1-2368, and NSF via IIS-1909816. JH is supported by the Centre for AI Fundamentals and the UKRI GenAI Hub.

\bibliographystyle{plainnat}
\bibliography{bib}


\crefalias{section}{appendix}
\crefalias{subsection}{appendix}
{
\hypersetup{linkcolor=blue}
\renewcommand{\contentsname}{Contents of Appendix}
\tableofcontents
}
\addtocontents{toc}
{\protect\setcounter{tocdepth}{2}}
\newpage
\appendix
\thispagestyle{empty}

\section{Formalization of density-related assumptions}\label{appendix:formal-assumption}

Our density-related assumptions in \Cref{sec:setup} can be more formally stated as follows:

\begin{assumption}
    There exists some base measure $\mu$ on $\cX$, with $\textnormal{supp}(\mu) = \cX$, such that $P_a \ll \mu$ and $\mu \ll P_a$ for all $a \in \cA_o$, where $\ll$ denotes the absolute continuity relation.
\end{assumption}
First, by the Radon-Nikodym theorem, the condition that $P^a << \mu$ ensures that each distribution $P^a$ has a density with respect to $\mu$, and that these densities are uniquely defined $\mu$-almost everywhere.
Second, this condition implies that $P^a(x) > 0$ for $\mu$-almost all $x \in \cX$.
Hence, our equality statements should be formally interpretated as holding over this \emph{equivalence class} of densities.


\section{Weighted composition: A causal representation learning perspective}\label{sec:perspectives}

In causal representation learning, one often considers \emph{causal representation models} as an inductive bias for compositional generalization.
In particular, one can define a \defword{causal representation model on $\cX$} as a tuple $M = (S, g)$, where $S$ is a \textit{structural causal model} on $\cZ = \bbR^d$, and $g\colon \cZ \to \cX$ is a diffeomorphism onto its image (see \citet[Definition 10.7]{varici2026causal}).
Most importantly for our purposes, the \defword{latent-space observational distribution} of $M = (S, g)$ has the form
\begin{equation*}
    Q^\obs(z) \defeq \prod_{i=1}^d Q^\obs(z_i \mid \pa_\cG(z_i)),
\end{equation*}
where $\cG$ is the \emph{causal DAG} associated with the structural causal model $S$, and the \defword{feature-space observational distribution} of $M$ is 
\begin{equation*}
    P^\obs(x) = Q^\obs(g^{-1}(x)) \cdot |\det J_{g^{-1}}(x) |.
\end{equation*}
Then, an \defword{intervention} $I$ on $M$, with \defword{targets} $T(I) \subseteq [d]$, is a collection of conditional distributions $( Q^I(z_i \mid \pa_\cG(z_i) )_{j \in T(I)}$, giving the \defword{latent-space interventional distribution}
\begin{equation*}
    Q^I(z) \defeq \prod_{i=1}^d Q^\obs(z_i \mid \pa_\cG(z_i))^{\kron_{i \not\in T(I)}} \cdot Q^I(z_i \mid \pa_\cG(z_i))^{\kron_{i \in T(I)}},
\end{equation*}
and the \defword{feature-space interventional distribution} $P^I(x) = Q^I(g^{-1}(x)) \cdot |\det J_{g^{-1}}(x) |$.
From these definitions, we see that
\begin{equation*}
    \frac{Q^I(z)}{Q^\obs(z)}
    =
    \prod_{i \in T(I)} \frac{Q^I(z_i \mid \pa_\cG(z_i))}{Q^\obs(z_i \mid \pa_\cG(z_i))},
\end{equation*}
and, since the Jacobian term cancels out,
\begin{equation}\label{eqn:density-ratio}
    \frac{P^I(x)}{P^\obs(x)}
    =
    \prod_{i \in T(I)} \frac{P^I(z_i \mid \pa_\cG(z_i))}{P^\obs(z_i \mid \pa_\cG(z_i))},
    \quad\text{where}\ z = g^{-1}(x).
\end{equation}

\paragraph{Perturbations as unknown-target interventions}
In \emph{interventional causal representation learning}, one typically assumes that each \emph{single perturbation} $a = \be_k$ for $k \in [K]$ corresponds to an intervention $I_a$ with unknown targets \citep{varici2024linear}, though possibly with restrictions on the size of the intervention \citep{ahuja2023interventional,squires2023linear,buchholz2023learning,zhang2023identifiability,varici2024general,varici2025score}.
In particular, to learn the model $M$ from such data, one must also learn a map $\rho\colon a \mapsto I_a$, mapping each perturbation to its representation as an intervention.

Representing \emph{double perturbations} as vectors $\be_{kk'} \defeq \be_k + \be_{k'}$ for $k \neq k'$, (a form of) the principle of \emph{independent causal mechanisms (ICM)} states that the map $\rho$ has a certain modularity property \citep{muller2021learning}.
In particular, if $\be_k$ and $\be_{k'}$ affect different causal mechanisms (i.e., $T(I_{\be_k}) \cap T(I_{\be_{k'}}) = \varnothing$), then the ICM principle encourages one to assume that $\rho(\be_{kk'}) = I_{\be_k} \cup I_{\be_{k'}}$, i.e., that the corresponding intervention $I_{\be_{kk'}}$ \emph{combines} the interventions $I_{\be_k} = ( Q^{I_{\be_k}}(z_i \mid \pa_\cG(z_i) )_{j \in T(I_{\be_k})}$ and $I_{\be_{k'}} = Q^{I_{\be_{k'}}}(z_i \mid \pa_\cG(z_i) )_{j \in T(I_{\be_{k'}})}$, without any changes to the interventional mechanisms.

\paragraph{Weighted composition}
For consistency with Section 2, let $P^a \defeq P^{I_a}$.
By \Cref{eqn:density-ratio} and the principle of independent mechanisms, we obtain 
\begin{align*}
    \frac{P^{\be_{kk'}}(x)}{P^\obs(x)}
    &=
    \prod_{i \in T\left(I_{\be_{kk'}}\right)} \frac{P^{\be_{kk'}}(z_i \mid \pa_\cG(z_i))}{P^\obs(z_i \mid \pa_\cG(z_i))},
    \tag{\Cref{eqn:density-ratio}}
    \\
    &=
    \left( \prod_{i \in T\left(I_{\be_{k}}\right)} \frac{P^{\be_{k}}(z_i \mid \pa_\cG(z_i))}{P^\obs(z_i \mid \pa_\cG(z_i))} \right)
    \cdot 
    \left( \prod_{i \in T\left(I_{\be_{k'}}\right)} \frac{P^{\be_{k'}}(z_i \mid \pa_\cG(z_i))}{P^\obs(z_i \mid \pa_\cG(z_i))} \right)
    \tag{ICM}
    \\
    &=
    \frac{P^{\be_k}(x)}{P^\obs(x)}
    \cdot
    \frac{P^{\be_{k'}}(x)}{P^\obs(x)},
    \tag{\Cref{eqn:density-ratio}}
\end{align*}
where we have used $z = g^{-1}(x)$ throughout.
Thus, rearranging and assuming $P^\bzero = P^\obs$, we obtain the desired composition: $P^{\be_{kk'}}(x) = P^{\be_k}(x) \cdot P^{\be_{k'}}(x) \cdot (P^{\bzero}(x))^{-1}$.
This argument naturally extends to combinations of \emph{several} perturbations with non-overlapping intervention targets.

Notably, such interventional distributions are generalizations of the \emph{Factorized Conditionals} condition from \citep{bradley2025mechanisms}: the assumption of mutually independent sets of variables corresponds to \emph{disconnected} groups of nodes in a causal graph, giving their Equation (7) as a special case where each intervention targets one of these disconnected groups of nodes.





\section{Feynman-Kac correctors}
\label{appendix:fkc}
In this section, we present the theoretical intuition behind Feynman-Kac Correctors (FKC) \citep{skreta2025feynman}.

Although the naïve denoising SDE uses a \emph{score} that matches $S^w_t$ at $t=0$, its marginal at $t=0$ is not $\bP^w$: the particle position at the endpoint depends on integrated dynamics over the entire trajectory, and $S^{w,\naive}_t$ is not correct at $t>0$. FKC addresses this inference-time approximation error by maintaining $K$ weighted particles $\{(x^{(k)}_t, \omega^{(k)}_t)\}_{k=1}^K$, with $\omega^{(k)}_t$ the log importance weight of particle $k$, such that the weighted empirical \emph{distribution} tracks a target along the entire denoising trajectory.

The natural target $\bP^w_t = \Noise_t(\bP^w)$ has an intractable score $S^w_t$, so FKC tracks instead the weighted composition of the noised marginals, $\Comp_w(\{P^a_t\}_a)$, whose score is the naïvely-composed $S^{w,\naive}_t$. This proxy agrees with $\bP^w_t$ at $t=0$ ($\Noise_0$ is the identity) but diverges from it at intermediate $t$, since noising and composition do not commute. Sampling correctness thus holds, at $t=0$ the proxy equals $\bP^w$ in distribution, so the weighted particles approximate $\bP^w$.

Determining the weight update to achieve this follows from comparing how the proposal $\tP^{w,\naive}_t$ and the proxy $\Comp_w(\{P^a_t\}_a)$ evolve in time. Both satisfy Fokker-Planck-type equations involving the drift divergence, score, and Laplacian of the log density. Subtracting one from the other, they leave a residual $g_t$ that the log-weights along a particle trajectory must account for:
\begin{equation}
    \d \omega_t(x_t)
    \;=\;
    \bar{g}_t(x_t)\,\d t,
    \qquad
    \bar{g}_t(x) \;:=\; g_t(x) - \mathbb{E}_{X_t \sim \Comp_w(\{P_t^a\}_a)}[g_t(X_t)],
    \label{eq:log-weight-update}
\end{equation}
where the centering enforces $\int \bar{g}_t(x)\,P^w_t(x)\,\d x = 0$,
preventing uniform divergence or collapse of all log-weights simultaneously.
For weighted compositions, $g_t$ is given in closed form by:

\begin{equation}
    g_t(x)
    \;=\;
    \left(1-\sum\nolimits_{a \in \cA_o}w_a \right) \langle \nabla_\xsym, u_t(x)\rangle 
    + 
    \frac{\sigma_t^2}{2}\!\left(
        \sum\nolimits_{a \in \Aobs} w_a\,\|S^a_t(x)\|^2
        -\|S^{w,\naive}_t(x)\|^2
    \right).
    \label{eq:g-sde}
\end{equation}

To utilize these weights $\omega_t^{(k)}$ to perform corrections, a Sequential Monte Carlo method is employed. During denoising, resampling is performed at each step based on the increment $\omega_t^{(k)}=g_t(x_t^{(k)})$ via systematic resampling proportional to $\exp\left(\omega_t^{(k)}\right)$ \citep{1521264}.
Note that when $K=1$, the ensemble consists of only a single particle, so resampling has no effect and running FKC is equivalent to running the na\"ive denoising SDE.

We also note how as an additional nicety of base-composed distributions \eqref{eq:base-composed-distribution}, the divergence term in \eqref{eq:g-sde} cancels to $0$ and can be ignored even for non-linear drifts. (For linear drifts, it is constant and can be ignored regardless).


The FKC framework extends naturally to denoising ODEs, with the same $g_t$ as in the SDE case. However, applying the instantiation of the FKC framework as described with Sequential Monte Carlo rejection sampling is less well-motivated for ODEs because their trajectories are deterministic. Because the reverse ODE lacks noise injection from the Wiener process, particle trajectories never stochastically diverge. When FKC resampling duplicates a high-weight particle, those identical particles will follow the exact same deterministic trajectory for the remainder of the reverse process. This means that resampling monotonically decreases the diversity of the particle swarm. Despite this limitation, if the continuity-equation residual is large, FKC can remain mathematically helpful to route the probability mass to the correct target coordinates.
\section{Proofs}\label{appendix:proofs}

\subsection{Proof of \texorpdfstring{\Cref{theorem:gaussian-integrability}}{Theorem~\ref{theorem:gaussian-integrability}}}\label{appendix:proof-gaussian-integrability}
\newtheorem{corollary}{Corollary}[theorem]
\begin{proof}
    Let $x \in \mathbb{R}^d$. The probability density function for each zero-mean multivariate Gaussian $P^a$ is given by:
    \[
        P^a(x) = \frac{1}{\sqrt{(2\pi)^d |\Sigma_a|}} \exp\left( -\frac{1}{2} x^\top \Sigma_a^{-1} x \right)
    \] 

    \noindent
    We wish to evaluate $f(x)$, the product of these densities raised to their respective weights $w_a$:
    \begin{align*}
        f(x) &= \prod_{a \in \Aobs} P^a(x)^{w_a}\\
        &= \prod_{a \in \Aobs} \left( (2\pi)^{-d/2} |\Sigma_a|^{-1/2} \exp\left( -\frac{1}{2} x^\top \Sigma_a^{-1} x \right) \right)^{w_a} \\
        &= \left( \prod_{a \in \Aobs} (2\pi)^{-d w_a / 2} |\Sigma_a|^{-w_a / 2} \right) \exp\left( -\frac{1}{2} \sum_{a \in \Aobs} w_a x^\top \Sigma_a^{-1} x \right)\\
        &= \left( \prod_{a \in \Aobs} (2\pi)^{-d w_a / 2} |\Sigma_a|^{-w_a / 2} \right) \exp\left( -\frac{1}{2} x^\top \left( \sum_{a \in \Aobs} w_a \Sigma_a^{-1} \right) x \right)
    \end{align*}

    \noindent
    Let $C$ be the scaling constant $C = \prod_{a \in \Aobs} (2\pi)^{-d w_a / 2} |\Sigma_a|^{-w_a / 2}$, and let $M = \sum_{a \in \Aobs} w_a \Sigma_a^{-1}$:
    \[
    f(x) = C \exp\left( -\frac{1}{2} x^\top M x \right)
    \]
    We can recognize $f(x) \propto \exp\left( -\frac{1}{2} x^\top M x \right)$ as the kernel of a multivariate Gaussian distribution with a mean of $0$ and precision $M$, which is well-known to be integrable if and only if $M \succ 0$ (that is, $M$ is positive definite). Therefore, assuming $M\succ 0$, normalizing $f(x)$ yields
    $P^w(x) = \cN(0, \Sigma^w)$ for 
    \begin{equation*}
        \Sigma^w = \left( \sum_{a \in \Aobs} w_a \Sigma_a^{-1} \right)^{-1}
    \end{equation*}
\end{proof}

\subsection{Proof of \texorpdfstring{\Cref{thm:commutative-gaussian-ode}}{Theorem~\ref{thm:commutative-gaussian-ode}}}\label{appendix:proof-commutative-gaussian-ode}

We first begin by proving a necessary Lemma.
\begin{lemma}[Log-derivative of a noised Gaussian covariance]\label{lem:log-deriv}
For $P^a = \cN(0, \Sigma_a)$, the noising marginal is $P^a_t = \cN(0, \Sigma^a_t)$ with $\Sigma^a_t = \alpha_t^2 \Sigma_a + \gamma_t^2 I$, and
\begin{equation*}
    \frac{d}{dt} \ln \Sigma^a_t = 2\mu_t I + \sigma_t^2 (\Sigma^a_t)^{-1}.
\end{equation*}
\end{lemma}

\begin{proof}
The first claim follows directly from the noising kernel $\Noise_t(x_0) = \cN(\alpha_t x_0; \gamma_t^2 I)$. For the second, both $\Sigma^a_t$ and
\begin{equation*}
    \dot\Sigma^a_t = 2\alpha_t \dot\alpha_t \Sigma_a + 2\gamma_t \dot\gamma_t I
\end{equation*}
are linear combinations of $\Sigma_a$ and $I$, so they commute, giving $\frac{d}{dt} \ln \Sigma^a_t = \dot\Sigma^a_t (\Sigma^a_t)^{-1}$. Using $\mu_t = \dot\alpha_t / \alpha_t$ and $\sigma_t^2 = 2\gamma_t \dot\gamma_t - 2\mu_t \gamma_t^2$,
\begin{equation*}
    \dot\Sigma^a_t = 2\mu_t (\alpha_t^2 \Sigma_a + \gamma_t^2 I) + \sigma_t^2 I = 2\mu_t \Sigma^a_t + \sigma_t^2 I.
\end{equation*}
Right-multiplying by $(\Sigma^a_t)^{-1}$ gives the result.
\end{proof}

Now we begin the proof of \Cref{thm:commutative-gaussian-ode}.
\begin{proof}
Since each $P^a$ is mean-zero Gaussian, $S^a_t(x) = -(\Sigma^a_t)^{-1} x$, so the na\"ively composed score is linear in $x$:
\begin{equation*}
    S^{w,\naive}_t(x) = \sum_{a \in \Aobs} w_a S^a_t(x) = -J^{w,\naive}_t \, x, \qquad J^{w,\naive}_t \defeq \sum_{a \in \Aobs} w_a (\Sigma^a_t)^{-1}.
\end{equation*}
Combined with $u_t(x) = \mu_t x$, the na\"ive denoising velocity field is also linear in $x$:
\begin{equation*}
    \tilde{v}^{w,\naive}_t(x) = A_t \, x, \qquad A_t \defeq \mu_t I + \tfrac{1}{2} \sigma_t^2 J^{w,\naive}_t,
\end{equation*}
with $A_t$ symmetric. The na\"ive denoising ODE \eqref{eqn:naive-denoising-ode}, integrated from $t = 1$ to $t = 0$, is therefore
\begin{equation*}
    \frac{dx_t}{dt} = A_t x_t, \qquad x_1 \sim \cN(0, I),
\end{equation*}
where the initial distribution reflects $\bP^w_1 = \cN(0, I)$.

\paragraph{Commutativity.} Note that every matrix appearing in this proof, $\Sigma^a_t$, $(\Sigma^a_t)^{-1}$, $A_t$, $\Sigma_t$, is built from $\{\Sigma_a\}_{a \in \Aobs}$ and $I$ by sums, products, inverses, and (for $\Sigma_t$) matrix exponentials. Since the $\Sigma_a$ pairwise commute by hypothesis, any two such matrices commute. In particular, $\Sigma_t$ commutes with $A_t$ and with $\dot\Sigma_t$, which justifies all matrix manipulations below and the usage of the identity $\frac{d}{dt}\ln \Sigma_t = \dot \Sigma_t \Sigma_t^{-1}$ previously.

\paragraph{Covariance evolution.} The covariance satisfies $\dot\Sigma_t = A_t \Sigma_t + \Sigma_t A_t^\top = 2 A_t \Sigma_t$ \citep[Equation~(6.2)]{Sarkka_Solin_2019}, so
\begin{equation*}
    \frac{d}{dt} \ln \Sigma_t = 2 A_t = 2\mu_t I + \sigma_t^2 \sum_{a \in \Aobs} w_a (\Sigma^a_t)^{-1}.
\end{equation*}
Using $\sum_{a \in \Aobs} w_a = 1$ to absorb the $2\mu_t I$ term into the sum, then applying Lemma~\ref{lem:log-deriv} termwise,
\begin{equation*}
    \frac{d}{dt} \ln \Sigma_t = \sum_{a \in \Aobs} w_a \bigl[ 2\mu_t I + \sigma_t^2 (\Sigma^a_t)^{-1} \bigr] = \sum_{a \in \Aobs} w_a \, \frac{d}{dt} \ln \Sigma^a_t.
\end{equation*}

\paragraph{Integration and boundary conditions.} Integrating from $t = 1$ to $t = 0$,
\begin{equation*}
    \ln \Sigma_0 - \ln \Sigma_1 = \sum_{a \in \Aobs} w_a \bigl( \ln \Sigma^a_0 - \ln \Sigma^a_1 \bigr).
\end{equation*}
The boundary values $\alpha_0 = 1, \gamma_0 = 0$ give $\Sigma^a_0 = \Sigma_a$, and $\alpha_1 = 0, \gamma_1 = 1$ give $\Sigma^a_1 = \Sigma_1 = I$, so $\ln \Sigma^a_1 = \ln \Sigma_1 = 0$. Hence
\begin{equation*}
    \ln \Sigma_0 = \sum_{a \in \Aobs} w_a \ln \Sigma_a,
    \qquad
    \Sigma_0 = \exp\!\Bigl( \sum_{a \in \Aobs} w_a \ln \Sigma_a \Bigr) = \prod_{a \in \Aobs} \Sigma_a^{w_a},
\end{equation*}
where the last equality uses that for commuting symmetric positive-definite matrices, $\exp(\ln A + \ln B) = AB$. 

\paragraph{Mean.} The mean $m_t$ satisfies $\dot m_t = A_t m_t$ with $m_1 = 0$, so $m_t \equiv 0$. Thus, $\tP^{w,\nvo} = \cN\bigl( 0, \prod_{a \in \Aobs} \Sigma_a^{w_a} \bigr)$.
\end{proof}

\subsection{Proof of \texorpdfstring{\Cref{thm:fkc-factorized-conditional}}{Theorem~\ref{thm:fkc-factorized-conditional}} (Feynman-Kac correctors under factorized conditionals)}\label{appendix:g-is-zero-in-FC}

To begin, we first prove a simple lemma that allows us to not worry about the time $t$.
\begin{lemma}[Coordinate-wise noising preserves factorization]
\label{lem:factorization-preserved}
If $P_0$ on $\mathbb{R}^n$ factorizes over a partition $[n] = \bigsqcup_\ell B_\ell$ as $P_0(x) = \prod_\ell P_0(x_{B_\ell})$, and the noising SDE acts coordinate-wise, then for every $t \in [0,1]$ the noising marginal $P_t$ factorizes over the same partition: $P_t(x) = \prod_\ell P_t(x_{B_\ell})$.
\end{lemma}

\begin{proof}
The coordinate-wise noising kernel satisfies $(\Noise_t(x_0))_{B_\ell} = \alpha_t (x_0)_{B_\ell} + \gamma_t Z_{B_\ell}$ with $Z_{B_\ell} \sim \cN(0, I_{|B_\ell|})$ independent across $\ell$.
If $X_0 \sim P_0$ has independent blocks by hypothesis, then the noised blocks $(X_t)_{B_\ell}$ are functions of disjoint independent inputs and remain mutually independent, so the joint density of $X_t$ factorizes block-wise.
\end{proof}

Now, we prove \Cref{thm:fkc-factorized-conditional}.
\begin{proof}
For convenience, let $\Aobs^- \defeq \Aobs \setminus \{0\}$, and let the distributions $\{P^0\} \cup \{P^a\}_{a \in \Aobs}$ on $\mathbb{R}^n$ denote the Factorized Conditional distributions.

Let $k \defeq |\Aobs^-|$. The weights satisfy
$w_0 + \sum_{a \in \Aobs^-} w_a = (1-k) + k = 1$, so the divergence term in
\eqref{eq:g-sde} vanishes. It remains to show
\begin{equation}\label{eq:fkc-residual}
    \Big\| w_0 S^0_t + \sum_{a \in \Aobs^-} w_a S^a_t \Big\|^2
    \;-\; \Big( w_0 \|S^0_t\|^2 + \sum_{a \in \Aobs^-} w_a \|S^a_t\|^2 \Big)
    \;=\; 0.
\end{equation}

\paragraph{Block decomposition of scores.}
The noising operator $\Noise_t$ \eqref{eq:noise} acts coordinate-wise because $\mu_t$ and $\sigma_t$ are scalar, so by Lemma~\ref{lem:factorization-preserved}, the factorization \eqref{eq:factorized-conditionals} is preserved at every
$t$. We drop the $t$ subscript hereafter.

Define the following vectors in $\mathbb{R}^n$, each supported on a single block of the partition:
\begin{align*}
    \tilde u_0 \;&\colon\; \text{supported on } M_0,
    \quad (\tilde u_0)_{M_0} \;=\; \nabla_{x_{M_0}} \log P^0(x_{M_0}); \\
    \tilde u_a \;&\colon\; \text{supported on } M_a,
    \quad (\tilde u_a)_{M_a} \;=\; \nabla_{x_{M_a}} \log P^0(x_{M_a}),
    \quad a \in \Aobs^-; \\
    \tilde v_a \;&\colon\; \text{supported on } M_a,
    \quad (\tilde v_a)_{M_a} \;=\; \nabla_{x_{M_a}} \log P^a(x_{M_a}),
    \quad a \in \Aobs^-.
\end{align*}
Differentiating the $P^0(x) = P^0(x|_{M_{0}}) \prod_{a\in \Aobs^-} P^0(x|_{M_{a}})$ term from \eqref{eq:factorized-conditionals} block-wise yields
\begin{equation}\label{eq:S0-decomp}
    S^0 \;=\; \tilde u_0 + \sum_{a \in \Aobs^-} \tilde u_a.
\end{equation}

For each $a \in \Aobs^-$, expanding
$\log P^a(x) = \log P^a(x_{M_a}) + \log P^0(x_{M_0}) + \sum_{a' \neq a} \log P^0(x_{M_{a'}})$
via \eqref{eq:factorized-conditionals} and differentiating gives
\begin{equation}\label{eq:Sa-decomp}
    S^a \;=\; \tilde v_a + \tilde u_0 + \sum_{a' \in \Aobs^- \setminus \{a\}} \tilde u_{a'}.
\end{equation}
The blocks $\{M_0\} \cup \{M_a\}_{a \in \Aobs^-}$ are all mutually disjoint, so vectors supported on different blocks are orthogonal in $\mathbb{R}^n$. The vectors $\tilde u_a$ and $\tilde v_a$ share support $M_a$ and are not in general orthogonal, but $\tilde u_a$ never appears in $S^a$, so this non-orthogonality plays no role in the calculations below.

\paragraph{Composed score.} Substituting \eqref{eq:S0-decomp} and
\eqref{eq:Sa-decomp} and collecting coefficients block-by-block,
\begin{align*}
    w_0 S^0 + \sum_{a \in \Aobs^-} w_a S^a
    &= (1-k)\!\left( \tilde u_0 + \sum_{a \in \Aobs^-} \tilde u_a \right)
       + \sum_{a \in \Aobs^-}\!\left( \tilde v_a + \tilde u_0
       + \sum_{a' \neq a} \tilde u_{a'} \right) \\
    &= \big[(1-k) + k\big]\, \tilde u_0
       + \sum_{a \in \Aobs^-}\!\big[(1-k) + (k-1)\big]\, \tilde u_a
       + \sum_{a \in \Aobs^-} \tilde v_a \\
    &= \tilde u_0 + \sum_{a \in \Aobs^-} \tilde v_a.
\end{align*}
The summands $\tilde u_0, \{\tilde v_a\}_{a \in \Aobs^-}$ live on the distinct blocks
$M_0, \{M_a\}_{a \in \Aobs^-}$, so by orthogonality
\begin{equation}\label{eq:norm-composed}
    \Big\| w_0 S^0 + \sum_{a \in \Aobs^-} w_a S^a \Big\|^2
    \;=\; \|\tilde u_0\|^2 + \sum_{a \in \Aobs^-} \|\tilde v_a\|^2.
\end{equation}

\paragraph{Sum of weighted squared norms.} Applying orthogonality to \eqref{eq:S0-decomp} and \eqref{eq:Sa-decomp} (using that $\tilde u_a$ does
not appear in $S^a$),
\begin{equation*}
    \|S^0\|^2 = \|\tilde u_0\|^2 + \sum_{a \in \Aobs^-} \|\tilde u_a\|^2,
    \qquad
    \|S^a\|^2 = \|\tilde v_a\|^2 + \|\tilde u_0\|^2 + \sum_{a' \neq a} \|\tilde u_{a'}\|^2.
\end{equation*}
We can then form the weighted combination and collect coefficients,
\begin{align*}
    w_0 \|S^0\|^2 + \sum_{a \in \Aobs^-}\! w_a \|S^a\|^2
    &= [(1-k) + k] \|\tilde u_0\|^2
    + \sum_{a \in \Aobs^-}\![(1-k) + (k-1)] \|\tilde u_a\|^2
    + \sum_{a \in \Aobs^-}\! \|\tilde v_a\|^2 \\
    &= \|\tilde u_0\|^2 + \sum_{a \in \Aobs^-} \|\tilde v_a\|^2.
\end{align*} 

Comparing with \eqref{eq:norm-composed} gives \eqref{eq:fkc-residual}, and combined with the vanishing divergence term, $g_t(x) = 0$ for all $x \in \mathbb{R}^n$ and $t \in [0,1]$.
\end{proof}

\begin{corollary}\label{cor:fkc-equivalent-naive}
Under the hypotheses of Theorem~\ref{thm:fkc-factorized-conditional}, FKC applied to the base composition \eqref{eq:base-composed-distribution} is exactly equivalent to running the na\"ive denoising SDE \eqref{eqn:naive-denoising-sde}.
\end{corollary}

\begin{proof}
By Theorem~\ref{thm:fkc-factorized-conditional}, $g_t \equiv 0$, hence $\bar g_t \equiv 0$. The log-weight update \eqref{eq:log-weight-update} then gives $\d \omega_t = 0$, so all log-weights remain at their initial value of $0$ throughout the trajectory. Systematic resampling proportional to $\exp(\omega_t^{(k)})$ therefore samples uniformly across particles at every step, leaving the ensemble unchanged. Consequently, the resulting marginal coincides exactly with that of the na\"ive denoising SDE.
\end{proof}

\section{Separate diffusion models}\label{appendix:sep-vs-conditional}
Instead of training a single diffusion model and composing its conditional distributions, it is possible to use separate diffusion models for the conditions, and compose in the same fashion via their product. This is desirable since, for instance, pre-trained diffusion models may possess different capabilities, and composing existing models may allow for greater control or improved performance on particular tasks than individually using any one pre-trained model.

To consider this approach, we re-run all of our experimental settings, individually training a separate diffusion model for each conditional distribution. Full results from these experiments are in \ref{subsec:sep1} and \ref{subsec:sep2}.

We observe that despite learning the underlying conditional distributions to the same level of in-distribution accuracy as conditional models, the compositions of these individual experts perform worse in practice than their conditional model counterpart. This effect is mild for 2D Gaussians (\Cref{tab:2d_separate_combined_results}), but is far more pronounced in image generation (\Cref{fig:room-grid-separate}), where couches often dominate outputs. Furthermore, in out-of-distribution image generation settings, high particle counts cause estimation error to accumulate so quickly that some samples degenerate entirely. The only exception, where separate models outperform a single conditional model, is the out-of-distribution Gaussian mixture setting. In this case, the separate models appear more resilient to performance degradation as the number of particles continues to increase past $4$. 
This evidence suggests that implicit regularization through the weight sharing of the conditional diffusion model may play an important role in mitigating out-of-distribution score estimation error and enabling effective weighted compositions.

\begin{table*}[t]
    \centering
    \caption{
        \textbf{The expected trends also hold for individual ``expert'' models, but overall performance is slightly worse and OOD estimation error accumulates more quickly.}
        The table follows the same conventions and set up as \Cref{tab:2d-combined_results}, but with individual models in place of a single conditional model.
    }
    \label{tab:2d_separate_combined_results}

    \begin{subtable}[t]{0.48\textwidth}
        \centering
        \caption{\textbf{Factorized conditionals + ID}}
        \begin{tabular}{@{}l|cccc@{}}
        \toprule
        $K \backslash N$ & 100 & 1000 & 10000 & analytic \\
        \midrule
        1 & $0.4443$ & $0.1339$ & $0.0933$ & $0.0370$ \\
        4 & $0.5546$ & $0.1217$ & $0.0765$ & $0.0363$ \\
        16 & $0.6076$ & $0.1177$ & $0.0763$ & $0.0372$ \\
        64 & $0.6267$ & $0.1140$ & $0.0743$ & $0.0365$ \\
        256 & $0.6364$ & $0.1196$ & $0.0786$ & $0.0359$ \\
        \bottomrule
        \end{tabular}
    \end{subtable}\hfill
    \begin{subtable}[t]{0.48\textwidth}
        \centering
        \caption{\textbf{Non-factorized conditionals + ID}}
        \begin{tabular}{@{}l|cccc@{}}
        \toprule
        $K \backslash N$ & 100 & 1000 & 10000 & analytic \\
        \midrule
        1 & $0.5553$ & $0.1437$ & $0.1106$ & $0.0673$ \\
        4 & $0.6524$ & $0.1213$ & $0.0781$ & $0.0389$ \\
        16 & $0.6930$ & $0.1148$ & $0.0711$ & $0.0356$ \\
        64 & $0.7034$ & $0.1147$ & $0.0705$ & $0.0358$ \\
        256 & $0.7110$ & $0.1157$ & $0.0737$ & $0.0350$ \\
        \bottomrule
        \end{tabular}
    \end{subtable}

    \vspace{1em}

    \begin{subtable}[t]{0.48\textwidth}
        \centering
        \caption{\textbf{Factorized conditionals + OOD}}
        \begin{tabular}{@{}l|cccc@{}}
        \toprule
        $K \backslash N$ & 100 & 1000 & 10000 & analytic \\
        \midrule
        1 & $1.1771$ & $0.8397$ & $0.7643$ & $0.1163$ \\
        4 & $2.2327$ & $2.0231$ & $2.0465$ & $0.1160$ \\
        16 & $3.1107$ & $3.2394$ & $3.4382$ & $0.1174$ \\
        64 & $3.5605$ & $3.9171$ & $4.2961$ & $0.1143$ \\
        256 & $3.8044$ & $4.2434$ & $4.7380$ & $0.1061$ \\
        \bottomrule
        \end{tabular}
    \end{subtable}\hfill
    \begin{subtable}[t]{0.48\textwidth}
        \centering
        \caption{\textbf{Non-factorized conditionals + OOD}}
        \begin{tabular}{@{}l|cccc@{}}
        \toprule
        $K \backslash N$ & 100 & 1000 & 10000 & analytic \\
        \midrule
        1 & $1.2463$ & $1.0431$ & $0.8929$ & $0.6418$ \\
        4 & $2.3105$ & $2.2133$ & $1.8170$ & $0.1704$ \\
        16 & $3.1794$ & $3.4853$ & $2.9968$ & $0.0905$ \\
        64 & $3.6653$ & $4.2021$ & $3.8652$ & $0.0821$ \\
        256 & $3.8937$ & $4.5325$ & $4.4165$ & $0.0815$ \\
        \bottomrule
        \end{tabular}
    \end{subtable}

\end{table*}

\begin{figure}[H]
    \centering
    \includegraphics[width=0.99\linewidth]{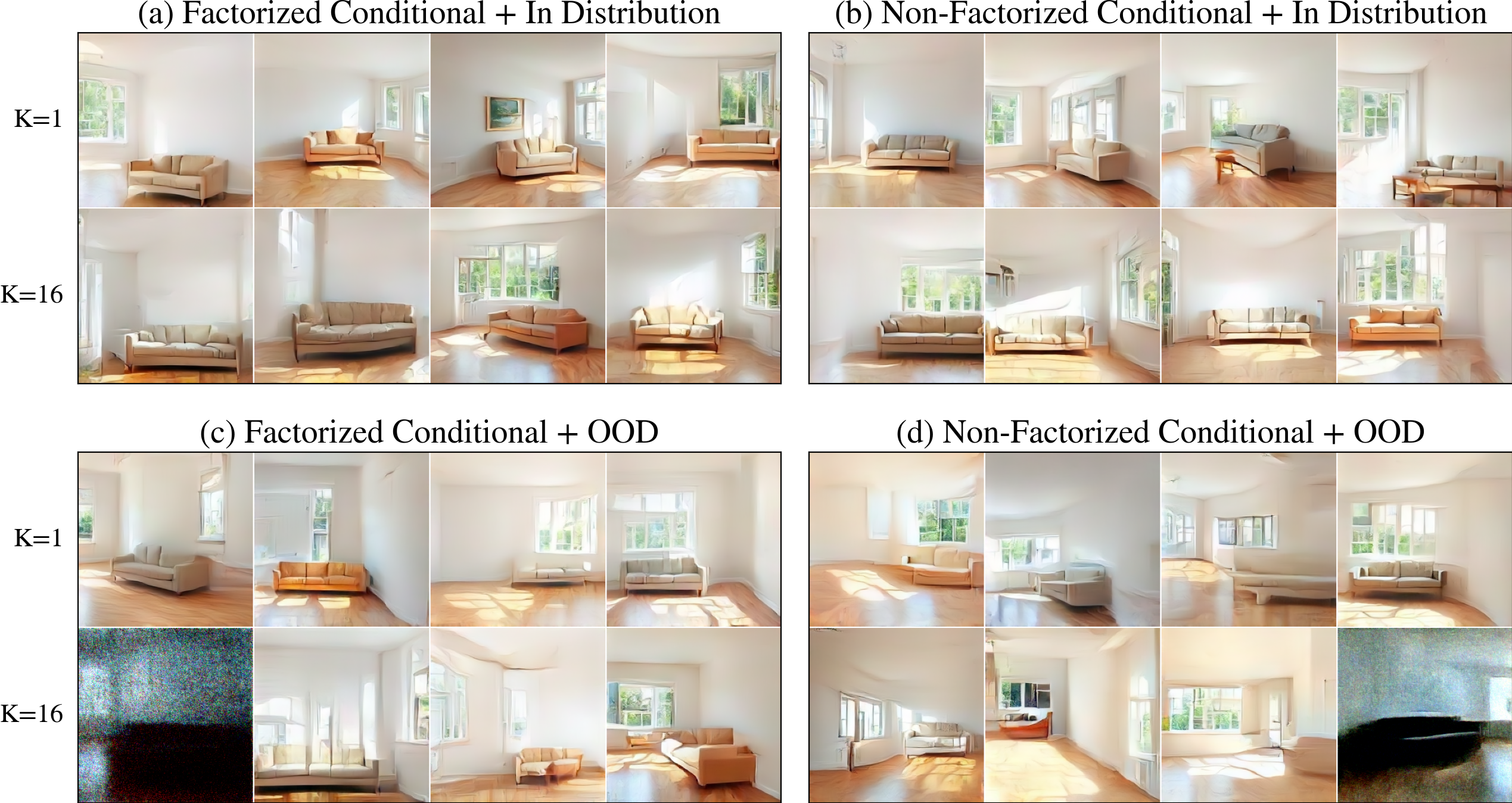}
    \caption{\textbf{Separate diffusion models compose less effectively than a single conditional diffusion model.} Set-up is identical to \Cref{fig:room-grid}, except composition is performed on separate models instead of on a single conditional model. 
    }
    \label{fig:room-grid-separate}
\end{figure}

\section{Experiment setup details}
\label{appendix:exp}
\subsection{2D Gaussian}\label{appendix:exp-details-2d-gaussian}
In this section, we provide additional experimental details. All experimental code is additionally provided at \url{https://github.com/DSoiffer/compositional-diffusion}.
All code is ran on the GPU, with either NVidia L40S or A6000s. The most expensive training code, which is for the image generation models, can be ran on a single GPU in around 9 hours. For each table, training and sampling across all configurations and across 30 runs consumes roughly 7 GPU hours. GPU-based computation is the primary computational bottleneck, memory and cpu requirements are minimal in comparison, though about 48GB of GPU memory is recommended for best performance.

\paragraph{Architecture.}
For each distribution we train a noise-prediction network $\epsilon_\theta(t, x)$ implemented as a 4-layer MLP with hidden width
512 and SiLU activations, predicting $\epsilon \in \mathbb{R}^2$ from the concatenation $[t, x]$. In the conditional setting a single network shared across $\{P^{a_1}, P^{a_2}, P^0\}$ takes $[t, x, \mathrm{onehot}(c)]$ as input, where $c \in \{0, 1, 2\}$ identifies the source distribution. For separate models, $c$ is omitted. The score is recovered as $s(t, x) = -\epsilon_\theta(t, x) / \sigma(t)$. We note that these architectures are highly overparameterized for the particular learning task, consistent with many modern setups.

\paragraph{Training objective.}
We use denoising score matching in $\epsilon$-prediction form \citep{denoising-diffusion-probabilistic-models}.

\paragraph{Optimizer and schedule.}
Adam~\citep{kingma2014adam} with learning rate $2 \times 10^{-4}$, default $(\beta_1, \beta_2) = (0.9, 0.999)$, $\varepsilon = 10^{-8}$, no weight decay, no learning-rate schedule. Batch size is set to $512$. Training is run for $20{,}000$ iterations in the single-shot pipeline and $20{,}000$ iterations per cell in the sweep, a dataset of $N$ points is created before training begins and sampled with replacement during training.

\paragraph{Noising schedule.}
Variance-preserving (VP) SDE~\citep{song2021scorebased} with linear $\beta(\tau) = \beta_{\min} + \tau (\beta_{\max} - \beta_{\min})$, $\beta_{\min} = 0.1$, $\beta_{\max} = 20$, on the forward time $\tau \in [0, 1]$. (Note that for implementation cleanliness, we use the opposite convention as the rest of our paper: $\tau = 1 - t$ so that $t = 0$ is pure noise and $t = 1$ is clean data.)

\paragraph{Sampling.}
We implement the Feynman--Kac corrector sampler~\citep{skreta2025feynman} running $N$ independent swarms of $K$ particles in parallel, performs systematic resampling within every swarm at every step, and at the final step draw one sample from each swarm's weighted ensemble. We use $500$ integration steps. We report metrics on $N = 5000$ output samples per configuration.

\paragraph{Sliced Wasserstein-2.}
Sliced $\mathcal{W}_2$ is computed via the Python Optimal Transport package \citep{flamary2021pot, flamary2024pot}, with $p = 2$ and $2000$ random projections per evaluation.

\paragraph{Maximum Mean Discrepancy.}
We report the unbiased $\mathrm{MMD}^{2}$ U-statistic with the Gaussian RBF kernel $k(x, y) = \exp\!\big(-\gamma \|x - y\|_2^{2}\big)$. $\gamma$ is set by the median heuristic on the empirical cross-pairwise Euclidean distances between the two sample sets, with $\gamma = 1 / (2 \mathrm{med}^{2})$.

\paragraph{Reference samples.}
Because the target is a closed-form Gaussian, we draw ground-truth samples directly via its diagonal covariance. Sliced $\mathcal{W}_2$ and $\mathrm{MMD}^{2}$ both use $5000$ ground-truth samples per evaluation.

\paragraph{Seeds.}
The base seed is $1$ and is applied at the start of each run. For multiple runs $N$, the run $r$ $(0 \le r < N)$ uses seed $1 + r$, which seeds every random component of the process. We average over $n_{\text{runs}} = 30$ independent runs and report mean $\pm$ standard deviation.

\paragraph{Clipping.} During the FKC denoisnig process, we introduce an optional \texttt{g\_clip} paramater, which clips the $g_t(x)$ increment \eqref{eq:g-sde} at every step to $[-\texttt{g\_clip}, \texttt{g\_clip}]$. Clipping generally reduces score estimation error by preventing weights from concentrating on erroneously high values, but can prevent score approximation error from decreasing as quickly as it would otherwise, particularly for oracle or well-learned models. In practice, we find this approach to be an effective method of trading off score estimation error for score approximation error. For 2D Gaussian and Mixture of Gaussians experiments, we set $\texttt{g\_clip}=15.0$, and leave it unset for image generation.

\subsection{Mixture of Gaussians}
\label{appendix:gmm-exp}
Unless otherwise explicitly noted, all neural network architectures, noising schedules, and optimization hyperparameters for the GMM experiments are identical to those described for the 2D Gaussian setting in Appendix~\ref{appendix:exp-details-2d-gaussian}.


\paragraph{Rejection sampling (In-Distribution)}
To obtain exact samples from the target distribution $P^w(x) \propto \frac{P^{a_1}(x)P^{a_2}(x)}{P^0(x)}$, we utilize an exact rejection sampling scheme. We designate $P^{a_1}(x)$ as the proposal distribution, while $P^{a_2}(x)$ and $P^0(x)$ serve as the numerator ratio factor and the denominator, respectively. Mathematically, the modes of $P^{a_2}$ are a subset of the modes of $P^0$. To compute the rejection bound analytically, we structure their implementations to share the exact same underlying GMM components (identical means and isotropic variances), where $P^{a_2}$ assigns a mixture weight of zero to any mode outside its subset. Because $P^{a_2}(x)$ and $P^0(x)$ now differ strictly in their mixture weight arrays (denoted as $w_{a_2}$ and $w_0$), the tight upper bound $M = \sup_x \frac{P^{a_2}(x)}{P^0(x)}$ can be computed exactly as $M = \max_k \frac{w_{a_2}[k]}{w_0[k]}$. Candidate samples are drawn from the proposal $x \sim P^{a_1}(x)$ and are subsequently accepted with probability $\frac{P^{a_2}(x)}{M \cdot P^0(x)}$.

\paragraph{Importance Sampling (Out-of-Distribution)}
In the out-of-distribution settings, the constituent GMMs do not share component means, rendering the analytical rejection bound intractable. For these cases, we employ Importance Sampling to obtain asymptotically exact samples. We compute the analytical product of all numerator GMMs to serve as the proposal distribution. We draw a heavily oversampled batch of candidate points from this proposal ($1000\times$ the target sample size) and assign each point a log-importance weight proportional to $-\sum_j \log p_j(x)$, representing the inverse effect of the denominator distributions. Finally, we perform systematic resampling based on these weights to extract $n$ unweighted samples approximating the target distribution.

\subsection{Objects in a room}\label{appendix:exp-details-room-objects}
For the sake of reproducibility, we produce full details here for our dataset construction and model training and sampling methodologies.

\paragraph{Dataset construction.}
To create our dataset of $256 \times 256$ images, we prompt the open-source text-to-image model \textsc{FLUX.1-schnell} \citep{flux2024,labs2025flux1kontextflowmatching} with a set of custom prompts. The prompt for generating an empty room is: \begin{quote}
    {A photograph of an empty living room with plain white walls and wooden floors. The room has a large window, and it is sunny outside. The room is completely empty. It contains no furniture, no decorations, no plants, and no other objects. Completely undecorated. Abandoned but clean. The photo is wide angle, showing the entire room and how it is empty.}
\end{quote}

\noindent
Class-conditioned prompts form variations on this, replacing the line
\begin{quote}
    The room is completely empty.
\end{quote}
as follows:
\begin{itemize}
    \item \textbf{couch:} The room is completely empty, except for a (couch:1.4)
    \item \textbf{coffee table:} The room is completely empty, except for a (coffee table:1.4)
    \item \textbf{framed painting:} The room is completely empty, except for a (framed painting:1.4) on the wall.
    \item \textbf{couch + coffee table}: The room is completely empty, except for a (couch:1.4) and a (coffee table:1.4).
    \item \textbf{couch + framed painting}: The room is completely empty, except for a (couch:1.4) and a (framed painting:1.4) on the wall.
\end{itemize}
Note that the numbers inside this prompt are interpreted by \textsc{FLUX.1-schnell} as token weights and not as tokens, affording us more control over its prompt adherence. 

After generating a set of candidate images, we manually label $1000$ images from each class as either \texttt{accept} if they meet the prompt's description and do not suffer from artifacting, and \texttt{reject} otherwise. We allow small incidental objects built-in to the room, such as radiators or small ceiling lights.

We embed the labeled images with \textsc{DINOv2-large} \citep{dinov2}. These embeddings are used to train a separate binary logistic regression classifier for each class to automatically accept or reject images. These classifiers are ran on the remaining samples (generating more as necessary) until we achieve $2000$ accepted images per class, which then forms our dataset.

\paragraph{Architecture.}
Class-conditional 2-D U-Net~\citep{ronneberger2015unet} (from the \texttt{diffusers} library~\citep{von-platen-etal-2022-diffusers}) with
input/output channels $3$, sample size $256$, two residual blocks per
level, channel multipliers $(128, 256, 256, 512)$, down-blocks
\texttt{(Down, Down, AttnDown, Down)}, up-blocks
\texttt{(Up, AttnUp, Up, Up)}, and a learned class embedding indexed
by the class (or condition) integer. The total number of class
embeddings equals the number of training labels for each task, i.e. $4$.

\paragraph{Forward process.}
Discrete DDPM~\citep{denoising-diffusion-probabilistic-models} with $T = 1000$ training timesteps and a linear $\beta$ schedule~\citep{nichol2021improvedddpm} via \texttt{diffusers.DDPMScheduler}. We use the $v$-prediction parametrization~\citep{salimans2022-v-prediction}, along with with terminal-SNR rescaling ~\citep{lin2024common-zero-snr} to remove all residual signal at the terminal noise level. To acquire scores from $v$-prediction to perform FKC sampling, we convert $v_t^a$ to a score via $S_t^a = -\epsilon^a_t/ \sigma(t)$ with $\epsilon^a_t = \sigma(t)\, x_t + \alpha(t)\, v^a_t$. We use $100$ integration steps during inference time.

\paragraph{Optimizer and schedule.}
AdamW~\citep{loshchilov2017-adamw} with learning rate $1 \times 10^{-4}$, $(\beta_1, \beta_2) = (0.9, 0.999)$, $\varepsilon = 10^{-8}$, no weight decay. Cosine learning-rate schedule with linear warmup over $500$ steps into $\text{epochs} \times \text{steps\_per\_epoch}$ total updates. Mixed-precision training in \texttt{bfloat16}. An exponential moving average of the model weights with decay $0.9999$ is maintained throughout training; sampling uses the EMA copy.

\paragraph{Data pipeline and hyperparameters.}
Per-image transform: \texttt{RandomHorizontalFlip}, \texttt{ToTensor},\texttt{Normalize(mean, std)} where $(\text{mean}, \text{std})$ is the fixed $(0.5, 0.5, 0.5)$ pair. We balance classes during training according to the mixture probabilities of each condition. Batch size $16$, $50$ epochs.

\section{Additional room images}\label{appendix:room-pics}

In this section, we provide additional samples from the compositions of learned diffusion models. These images are shown in \Cref{fig:room-rows}.

\begin{figure}[H]
    \centering
    \includegraphics[width=0.99\linewidth]{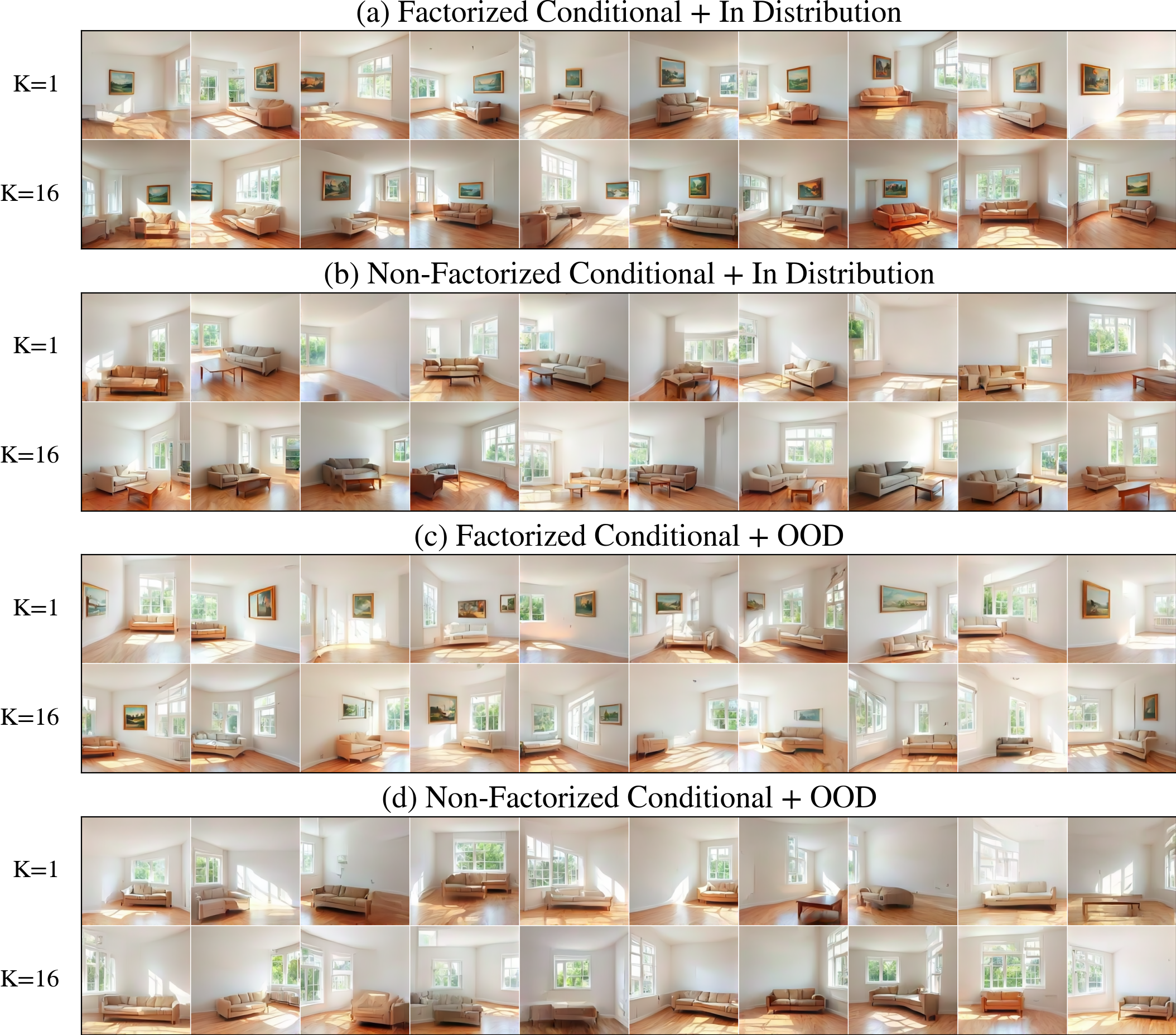}
    \caption{Additional samples from composed conditional diffusion models. Setup and trends are the same as reported in \Cref{fig:room-grid}.}
    \label{fig:room-rows}
\end{figure}

\section{Tables}\label{appendix:tables}

\subsection{Two-dimensional Gaussian (conditional model)}

\begin{table}[H]
\centering
\caption{\textbf{Factorized Conditionals + In-Distribution Composition.} $P^{a_1}=\mathcal{N}(0,\,\mathrm{diag}(10,1))$, $P^{a_2}=\mathcal{N}(0,\,\mathrm{diag}(1,10))$, $P^0=\mathcal{N}(0,\,10I)$, $\bP^w=\mathcal{N}(0,\,I)$): $\mathrm{SW}_2$ (mean $\pm$ std over 30 runs)}
\label{tab:sweep_sw2_1}
\begin{tabular}{l|cccc}
\toprule
$K \backslash N$ & N=100 & N=1000 & N=10000 & analytical \\
\midrule
K=1 & $0.3537 \pm 0.0859$ & $0.1089 \pm 0.0266$ & $0.0747 \pm 0.0231$ & $0.0354 \pm 0.0056$ \\
K=4 & $0.4276 \pm 0.1379$ & $0.0974 \pm 0.0308$ & $0.0750 \pm 0.0208$ & $0.0376 \pm 0.0054$ \\
K=16 & $0.4894 \pm 0.1747$ & $0.0956 \pm 0.0271$ & $0.0772 \pm 0.0226$ & $0.0351 \pm 0.0054$ \\
K=64 & $0.5136 \pm 0.1882$ & $0.0937 \pm 0.0269$ & $0.0758 \pm 0.0223$ & $0.0356 \pm 0.0044$ \\
K=256 & $0.5261 \pm 0.1913$ & $0.1010 \pm 0.0274$ & $0.0777 \pm 0.0237$ & $0.0346 \pm 0.0057$ \\
\bottomrule
\end{tabular}
\end{table}

\begin{table}[H]
\centering
\caption{\textbf{Factorized Conditionals + In-Distribution Composition.} $P^{a_1}=\mathcal{N}(0,\,\mathrm{diag}(10,1))$, $P^{a_2}=\mathcal{N}(0,\,\mathrm{diag}(1,10))$, $P^0=\mathcal{N}(0,\,10I)$, $\bP^w=\mathcal{N}(0,\,I)$): $\mathrm{MMD}^2$ (mean $\pm$ std over 30 runs)}
\label{tab:sweep_mmd2_1}
\begin{tabular}{l|cccc}
\toprule
$K \backslash N$ & N=100 & N=1000 & N=10000 & analytical \\
\midrule
K=1 & $0.0252 \pm 0.0173$ & $0.0025 \pm 0.0014$ & $0.0011 \pm 0.0009$ & $-0.0000 \pm 0.0001$ \\
K=4 & $0.0449 \pm 0.0397$ & $0.0021 \pm 0.0015$ & $0.0010 \pm 0.0008$ & $0.0000 \pm 0.0001$ \\
K=16 & $0.0604 \pm 0.0526$ & $0.0019 \pm 0.0014$ & $0.0012 \pm 0.0009$ & $-0.0000 \pm 0.0001$ \\
K=64 & $0.0657 \pm 0.0561$ & $0.0018 \pm 0.0015$ & $0.0012 \pm 0.0010$ & $-0.0000 \pm 0.0001$ \\
K=256 & $0.0684 \pm 0.0564$ & $0.0021 \pm 0.0015$ & $0.0012 \pm 0.0011$ & $-0.0000 \pm 0.0001$ \\
\bottomrule
\end{tabular}
\end{table}

\begin{table}[H]
\centering
\caption{\textbf{Non-Factorized Conditionals + In-Distribution Composition.} $P^{a_1}=\mathcal{N}(0,\,\mathrm{diag}(10,1))$, $P^{a_2}=\mathcal{N}(0,\,\mathrm{diag}(1,10))$, $P^0=\mathcal{N}(0,\,20I)$, $\bP^w=\mathcal{N}(0,\,\frac{20}{21}I)$): $\mathrm{SW}_2$ (mean $\pm$ std over 30 runs)}
\label{tab:sweep_sw2_3}
\begin{tabular}{l|cccc}
\toprule
$K \backslash N$ & N=100 & N=1000 & N=10000 & analytical \\
\midrule
K=1 & $0.4454 \pm 0.1490$ & $0.1234 \pm 0.0310$ & $0.0957 \pm 0.0240$ & $0.0655 \pm 0.0085$ \\
K=4 & $0.5331 \pm 0.1920$ & $0.0969 \pm 0.0293$ & $0.0751 \pm 0.0219$ & $0.0412 \pm 0.0045$ \\
K=16 & $0.5914 \pm 0.2111$ & $0.0901 \pm 0.0253$ & $0.0750 \pm 0.0227$ & $0.0353 \pm 0.0061$ \\
K=64 & $0.6111 \pm 0.2162$ & $0.0905 \pm 0.0237$ & $0.0740 \pm 0.0230$ & $0.0357 \pm 0.0051$ \\
K=256 & $0.6216 \pm 0.2201$ & $0.0963 \pm 0.0270$ & $0.0793 \pm 0.0234$ & $0.0333 \pm 0.0043$ \\
\bottomrule
\end{tabular}
\end{table}

\begin{table}[H]
\centering
\caption{\textbf{Non-Factorized Conditionals + In-Distribution Composition.} $P^{a_1}=\mathcal{N}(0,\,\mathrm{diag}(10,1))$, $P^{a_2}=\mathcal{N}(0,\,\mathrm{diag}(1,10))$, $P^0=\mathcal{N}(0,\,20I)$, $\bP^w=\mathcal{N}(0,\,\frac{20}{21}I)$): $\mathrm{MMD}^2$ (mean $\pm$ std over 30 runs)}
\label{tab:sweep_mmd2_3}
\begin{tabular}{l|cccc}
\toprule
$K \backslash N$ & N=100 & N=1000 & N=10000 & analytical \\
\midrule
K=1 & $0.0507 \pm 0.0412$ & $0.0038 \pm 0.0021$ & $0.0020 \pm 0.0011$ & $0.0007 \pm 0.0003$ \\
K=4 & $0.0803 \pm 0.0628$ & $0.0022 \pm 0.0016$ & $0.0011 \pm 0.0009$ & $0.0001 \pm 0.0001$ \\
K=16 & $0.0993 \pm 0.0713$ & $0.0018 \pm 0.0012$ & $0.0012 \pm 0.0010$ & $0.0000 \pm 0.0001$ \\
K=64 & $0.1057 \pm 0.0748$ & $0.0018 \pm 0.0013$ & $0.0012 \pm 0.0011$ & $0.0000 \pm 0.0001$ \\
K=256 & $0.1093 \pm 0.0764$ & $0.0021 \pm 0.0016$ & $0.0013 \pm 0.0012$ & $-0.0000 \pm 0.0001$ \\
\bottomrule
\end{tabular}
\end{table}

\begin{table}[H]
\centering
\caption{\textbf{Factorized Conditionals + Out-Of-Distribution Composition.} $P^{a_1}=\mathcal{N}(0,\,\mathrm{diag}(10,1))$, $P^{a_2}=\mathcal{N}(0,\,\mathrm{diag}(1,10))$, $P^0=\mathcal{N}(0,\,I)$, $\bP^w=\mathcal{N}(0,\,10I)$): $\mathrm{SW}_2$ (mean $\pm$ std over 30 runs)}
\label{tab:sweep_sw2_2}
\begin{tabular}{l|cccc}
\toprule
$K \backslash N$ & N=100 & N=1000 & N=10000 & analytical \\
\midrule
K=1 & $0.9658 \pm 0.2036$ & $0.5577 \pm 0.1655$ & $0.3948 \pm 0.1459$ & $0.1122 \pm 0.0189$ \\
K=4 & $2.0688 \pm 0.6346$ & $1.6885 \pm 0.4857$ & $1.1773 \pm 0.4160$ & $0.1130 \pm 0.0129$ \\
K=16 & $2.9177 \pm 0.9132$ & $2.7214 \pm 0.8497$ & $1.9699 \pm 0.6611$ & $0.1133 \pm 0.0228$ \\
K=64 & $3.3210 \pm 1.0273$ & $3.3348 \pm 1.0738$ & $2.4983 \pm 0.8347$ & $0.1124 \pm 0.0147$ \\
K=256 & $3.5097 \pm 1.0678$ & $3.6426 \pm 1.2007$ & $2.8210 \pm 0.9801$ & $0.1140 \pm 0.0158$ \\
\bottomrule
\end{tabular}
\end{table}

\begin{table}[H]
\centering
\caption{\textbf{Factorized Conditionals + Out-Of-Distribution Composition.} $P^{a_1}=\mathcal{N}(0,\,\mathrm{diag}(10,1))$, $P^{a_2}=\mathcal{N}(0,\,\mathrm{diag}(1,10))$, $P^0=\mathcal{N}(0,\,I)$, $\bP^w=\mathcal{N}(0,\,10I)$): $\mathrm{MMD}^2$ (mean $\pm$ std over 30 runs)}
\label{tab:sweep_mmd2_2}
\begin{tabular}{l|cccc}
\toprule
$K \backslash N$ & N=100 & N=1000 & N=10000 & analytical \\
\midrule
K=1 & $0.0135 \pm 0.0067$ & $0.0042 \pm 0.0025$ & $0.0025 \pm 0.0019$ & $-0.0000 \pm 0.0001$ \\
K=4 & $0.0872 \pm 0.0513$ & $0.0541 \pm 0.0264$ & $0.0271 \pm 0.0160$ & $-0.0000 \pm 0.0001$ \\
K=16 & $0.1677 \pm 0.0977$ & $0.1294 \pm 0.0694$ & $0.0695 \pm 0.0373$ & $0.0000 \pm 0.0002$ \\
K=64 & $0.2039 \pm 0.1171$ & $0.1778 \pm 0.0930$ & $0.1042 \pm 0.0538$ & $-0.0000 \pm 0.0001$ \\
K=256 & $0.2210 \pm 0.1253$ & $0.2013 \pm 0.1010$ & $0.1269 \pm 0.0685$ & $0.0000 \pm 0.0001$ \\
\bottomrule
\end{tabular}
\end{table}

\begin{table}[H]
\centering
\caption{\textbf{Non-Factorized Conditionals + Out-Of-Distribution Composition.} $P^{a_1}=\mathcal{N}(0,\,\mathrm{diag}(10,1))$, $P^{a_2}=\mathcal{N}(0,\,\mathrm{diag}(1,10))$, $P^0=\mathcal{N}(0,\,1.1I)$, $\bP^w=\mathcal{N}(0,\,\frac{110}{21}I)$): $\mathrm{SW}_2$ (mean $\pm$ std over 30 runs)}
\label{tab:sweep_sw2_4}
\begin{tabular}{l|cccc}
\toprule
$K \backslash N$ & N=100 & N=1000 & N=10000 & analytical \\
\midrule
K=1 & $1.0957 \pm 0.3179$ & $0.9154 \pm 0.2265$ & $0.8518 \pm 0.2079$ & $0.6375 \pm 0.0257$ \\
K=4 & $2.1611 \pm 0.6579$ & $1.8306 \pm 0.5788$ & $1.3763 \pm 0.5356$ & $0.1764 \pm 0.0222$ \\
K=16 & $3.0011 \pm 0.8826$ & $2.8830 \pm 0.9744$ & $2.0117 \pm 0.8550$ & $0.0981 \pm 0.0197$ \\
K=64 & $3.4224 \pm 1.0159$ & $3.6247 \pm 1.2177$ & $2.5011 \pm 1.1307$ & $0.0831 \pm 0.0128$ \\
K=256 & $3.6076 \pm 1.0687$ & $3.9971 \pm 1.3265$ & $2.8322 \pm 1.3633$ & $0.0820 \pm 0.0147$ \\
\bottomrule
\end{tabular}
\end{table}

\begin{table}[H]
\centering
\caption{\textbf{Non-Factorized Conditionals + Out-Of-Distribution Composition.} $P^{a_1}=\mathcal{N}(0,\,\mathrm{diag}(10,1))$, $P^{a_2}=\mathcal{N}(0,\,\mathrm{diag}(1,10))$, $P^0=\mathcal{N}(0,\,1.1I)$, $\bP^w=\mathcal{N}(0,\,\frac{110}{21}I)$): $\mathrm{MMD}^2$ (mean $\pm$ std over 30 runs)}
\label{tab:sweep_mmd2_4}
\begin{tabular}{l|cccc}
\toprule
$K \backslash N$ & N=100 & N=1000 & N=10000 & analytical \\
\midrule
K=1 & $0.0300 \pm 0.0148$ & $0.0189 \pm 0.0079$ & $0.0172 \pm 0.0063$ & $0.0117 \pm 0.0010$ \\
K=4 & $0.1267 \pm 0.0599$ & $0.0770 \pm 0.0358$ & $0.0438 \pm 0.0271$ & $0.0009 \pm 0.0003$ \\
K=16 & $0.2241 \pm 0.1033$ & $0.1758 \pm 0.0890$ & $0.0862 \pm 0.0539$ & $0.0001 \pm 0.0002$ \\
K=64 & $0.2689 \pm 0.1243$ & $0.2498 \pm 0.1201$ & $0.1215 \pm 0.0794$ & $-0.0000 \pm 0.0001$ \\
K=256 & $0.2866 \pm 0.1310$ & $0.2853 \pm 0.1313$ & $0.1455 \pm 0.0996$ & $-0.0000 \pm 0.0001$ \\
\bottomrule
\end{tabular}
\end{table}

\subsection{Two-dimensional Gaussian (separate models)}
\label{subsec:sep1}

\begin{table}[H]
\centering
\caption{\textbf{Factorized Conditionals + In-Distribution Composition.} Separate models, $P^{a_1}=\mathcal{N}(0,\,\mathrm{diag}(10,1))$, $P^{a_2}=\mathcal{N}(0,\,\mathrm{diag}(1,10))$, $P^0=\mathcal{N}(0,\,10I)$, $\bP^w=\mathcal{N}(0,\,I)$): $\mathrm{SW}_2$ (mean $\pm$ std over 30 runs)}
\label{tab:sweep_sw2_sep1}
\begin{tabular}{l|cccc}
\toprule
$K \backslash N$ & N=100 & N=1000 & N=10000 & analytical \\
\midrule
K=1 & $0.4443 \pm 0.0893$ & $0.1339 \pm 0.0375$ & $0.0933 \pm 0.0327$ & $0.0370 \pm 0.0051$ \\
K=4 & $0.5546 \pm 0.1553$ & $0.1217 \pm 0.0304$ & $0.0765 \pm 0.0259$ & $0.0363 \pm 0.0057$ \\
K=16 & $0.6076 \pm 0.1804$ & $0.1177 \pm 0.0284$ & $0.0763 \pm 0.0224$ & $0.0372 \pm 0.0067$ \\
K=64 & $0.6267 \pm 0.1911$ & $0.1140 \pm 0.0327$ & $0.0743 \pm 0.0278$ & $0.0365 \pm 0.0048$ \\
K=256 & $0.6364 \pm 0.1947$ & $0.1196 \pm 0.0291$ & $0.0786 \pm 0.0242$ & $0.0359 \pm 0.0062$ \\
\bottomrule
\end{tabular}
\end{table}

\begin{table}[H]
\centering
\caption{\textbf{Factorized Conditionals + In-Distribution Composition.} Separate models, $P^{a_1}=\mathcal{N}(0,\,\mathrm{diag}(10,1))$, $P^{a_2}=\mathcal{N}(0,\,\mathrm{diag}(1,10))$, $P^0=\mathcal{N}(0,\,10I)$, $\bP^w=\mathcal{N}(0,\,I)$): $\mathrm{MMD}^2$ (mean $\pm$ std over 30 runs)}
\label{tab:sweep_mmd2_sep1}
\begin{tabular}{l|cccc}
\toprule
$K \backslash N$ & N=100 & N=1000 & N=10000 & analytical \\
\midrule
K=1 & $0.0323 \pm 0.0206$ & $0.0041 \pm 0.0026$ & $0.0019 \pm 0.0016$ & $0.0000 \pm 0.0001$ \\
K=4 & $0.0647 \pm 0.0485$ & $0.0033 \pm 0.0021$ & $0.0012 \pm 0.0012$ & $0.0000 \pm 0.0001$ \\
K=16 & $0.0817 \pm 0.0584$ & $0.0030 \pm 0.0017$ & $0.0011 \pm 0.0008$ & $0.0000 \pm 0.0001$ \\
K=64 & $0.0876 \pm 0.0624$ & $0.0028 \pm 0.0019$ & $0.0011 \pm 0.0011$ & $-0.0000 \pm 0.0001$ \\
K=256 & $0.0903 \pm 0.0638$ & $0.0031 \pm 0.0019$ & $0.0012 \pm 0.0011$ & $0.0000 \pm 0.0001$ \\
\bottomrule
\end{tabular}
\end{table}

\begin{table}[H]
\centering
\caption{\textbf{Non-Factorized Conditionals + In-Distribution Composition.} Separate models, $P^{a_1}=\mathcal{N}(0,\,\mathrm{diag}(10,1))$, $P^{a_2}=\mathcal{N}(0,\,\mathrm{diag}(1,10))$, $P^0=\mathcal{N}(0,\,20I)$, $\bP^w=\mathcal{N}(0,\,\frac{20}{21}I)$): $\mathrm{SW}_2$ (mean $\pm$ std over 30 runs)}
\label{tab:sweep_sw2_sep3}
\begin{tabular}{l|cccc}
\toprule
$K \backslash N$ & N=100 & N=1000 & N=10000 & analytical \\
\midrule
K=1 & $0.5553 \pm 0.1378$ & $0.1437 \pm 0.0343$ & $0.1106 \pm 0.0300$ & $0.0673 \pm 0.0073$ \\
K=4 & $0.6524 \pm 0.1708$ & $0.1213 \pm 0.0313$ & $0.0781 \pm 0.0255$ & $0.0389 \pm 0.0079$ \\
K=16 & $0.6930 \pm 0.1783$ & $0.1148 \pm 0.0283$ & $0.0711 \pm 0.0194$ & $0.0356 \pm 0.0057$ \\
K=64 & $0.7034 \pm 0.1815$ & $0.1147 \pm 0.0322$ & $0.0705 \pm 0.0248$ & $0.0358 \pm 0.0049$ \\
K=256 & $0.7110 \pm 0.1810$ & $0.1157 \pm 0.0320$ & $0.0737 \pm 0.0226$ & $0.0350 \pm 0.0063$ \\
\bottomrule
\end{tabular}
\end{table}

\begin{table}[H]
\centering
\caption{\textbf{Non-Factorized Conditionals + In-Distribution Composition.} Separate models, $P^{a_1}=\mathcal{N}(0,\,\mathrm{diag}(10,1))$, $P^{a_2}=\mathcal{N}(0,\,\mathrm{diag}(1,10))$, $P^0=\mathcal{N}(0,\,20I)$, $\bP^w=\mathcal{N}(0,\,\frac{20}{21}I)$): $\mathrm{MMD}^2$ (mean $\pm$ std over 30 runs)}
\label{tab:sweep_mmd2_sep3}
\begin{tabular}{l|cccc}
\toprule
$K \backslash N$ & N=100 & N=1000 & N=10000 & analytical \\
\midrule
K=1 & $0.0640 \pm 0.0471$ & $0.0053 \pm 0.0029$ & $0.0030 \pm 0.0018$ & $0.0007 \pm 0.0002$ \\
K=4 & $0.0979 \pm 0.0708$ & $0.0037 \pm 0.0023$ & $0.0014 \pm 0.0012$ & $0.0001 \pm 0.0002$ \\
K=16 & $0.1123 \pm 0.0763$ & $0.0031 \pm 0.0020$ & $0.0010 \pm 0.0007$ & $0.0000 \pm 0.0001$ \\
K=64 & $0.1159 \pm 0.0789$ & $0.0032 \pm 0.0022$ & $0.0010 \pm 0.0010$ & $-0.0000 \pm 0.0001$ \\
K=256 & $0.1185 \pm 0.0796$ & $0.0032 \pm 0.0024$ & $0.0011 \pm 0.0010$ & $0.0000 \pm 0.0001$ \\
\bottomrule
\end{tabular}
\end{table}

\begin{table}[H]
\centering
\caption{\textbf{Factorized Conditionals + Out-Of-Distribution Composition.} Separate models, $P^{a_1}=\mathcal{N}(0,\,\mathrm{diag}(10,1))$, $P^{a_2}=\mathcal{N}(0,\,\mathrm{diag}(1,10))$, $P^0=\mathcal{N}(0,\,I)$, $\bP^w=\mathcal{N}(0,\,10I)$): $\mathrm{SW}_2$ (mean $\pm$ std over 30 runs)}
\label{tab:sweep_sw2_sep2}
\begin{tabular}{l|cccc}
\toprule
$K \backslash N$ & N=100 & N=1000 & N=10000 & analytical \\
\midrule
K=1 & $1.1771 \pm 0.3766$ & $0.8397 \pm 0.2398$ & $0.7643 \pm 0.3888$ & $0.1163 \pm 0.0150$ \\
K=4 & $2.2327 \pm 0.9720$ & $2.0231 \pm 0.6790$ & $2.0465 \pm 0.9039$ & $0.1160 \pm 0.0205$ \\
K=16 & $3.1107 \pm 1.4567$ & $3.2394 \pm 0.9808$ & $3.4382 \pm 1.6403$ & $0.1174 \pm 0.0206$ \\
K=64 & $3.5605 \pm 1.6860$ & $3.9171 \pm 1.0955$ & $4.2961 \pm 2.1283$ & $0.1143 \pm 0.0150$ \\
K=256 & $3.8044 \pm 1.8279$ & $4.2434 \pm 1.2064$ & $4.7380 \pm 2.4489$ & $0.1061 \pm 0.0211$ \\
\bottomrule
\end{tabular}
\end{table}

\begin{table}[H]
\centering
\caption{\textbf{Factorized Conditionals + Out-Of-Distribution Composition.} Separate models, $P^{a_1}=\mathcal{N}(0,\,\mathrm{diag}(10,1))$, $P^{a_2}=\mathcal{N}(0,\,\mathrm{diag}(1,10))$, $P^0=\mathcal{N}(0,\,I)$, $\bP^w=\mathcal{N}(0,\,10I)$): $\mathrm{MMD}^2$ (mean $\pm$ std over 30 runs)}
\label{tab:sweep_mmd2_sep2}
\begin{tabular}{l|cccc}
\toprule
$K \backslash N$ & N=100 & N=1000 & N=10000 & analytical \\
\midrule
K=1 & $0.0179 \pm 0.0090$ & $0.0083 \pm 0.0042$ & $0.0049 \pm 0.0029$ & $0.0000 \pm 0.0001$ \\
K=4 & $0.0903 \pm 0.0639$ & $0.0678 \pm 0.0387$ & $0.0655 \pm 0.0354$ & $0.0000 \pm 0.0001$ \\
K=16 & $0.1727 \pm 0.1192$ & $0.1600 \pm 0.0751$ & $0.1777 \pm 0.1020$ & $0.0000 \pm 0.0001$ \\
K=64 & $0.2171 \pm 0.1387$ & $0.2174 \pm 0.0922$ & $0.2428 \pm 0.1353$ & $-0.0000 \pm 0.0001$ \\
K=256 & $0.2403 \pm 0.1440$ & $0.2445 \pm 0.1040$ & $0.2707 \pm 0.1440$ & $-0.0000 \pm 0.0001$ \\
\bottomrule
\end{tabular}
\end{table}

\begin{table}[H]
\centering
\caption{\textbf{Non-Factorized Conditionals + Out-Of-Distribution Composition.} Separate models, $P^{a_1}=\mathcal{N}(0,\,\mathrm{diag}(10,1))$, $P^{a_2}=\mathcal{N}(0,\,\mathrm{diag}(1,10))$, $P^0=\mathcal{N}(0,\,1.1I)$, $\bP^w=\mathcal{N}(0,\,\frac{110}{21}I)$): $\mathrm{SW}_2$ (mean $\pm$ std over 30 runs)}
\label{tab:sweep_sw2_sep4}
\begin{tabular}{l|cccc}
\toprule
$K \backslash N$ & N=100 & N=1000 & N=10000 & analytical \\
\midrule
K=1 & $1.2463 \pm 0.4880$ & $1.0431 \pm 0.3266$ & $0.8929 \pm 0.4140$ & $0.6418 \pm 0.0308$ \\
K=4 & $2.3105 \pm 0.9665$ & $2.2133 \pm 0.8139$ & $1.8170 \pm 0.9345$ & $0.1704 \pm 0.0218$ \\
K=16 & $3.1794 \pm 1.4174$ & $3.4853 \pm 1.2953$ & $2.9968 \pm 1.5100$ & $0.0905 \pm 0.0153$ \\
K=64 & $3.6653 \pm 1.6599$ & $4.2021 \pm 1.5352$ & $3.8652 \pm 1.8677$ & $0.0821 \pm 0.0100$ \\
K=256 & $3.8937 \pm 1.7769$ & $4.5325 \pm 1.6601$ & $4.4165 \pm 2.1453$ & $0.0815 \pm 0.0139$ \\
\bottomrule
\end{tabular}
\end{table}

\begin{table}[H]
\centering
\caption{\textbf{Non-Factorized Conditionals + Out-Of-Distribution Composition.} Separate models, $P^{a_1}=\mathcal{N}(0,\,\mathrm{diag}(10,1))$, $P^{a_2}=\mathcal{N}(0,\,\mathrm{diag}(1,10))$, $P^0=\mathcal{N}(0,\,1.1I)$, $\bP^w=\mathcal{N}(0,\,\frac{110}{21}I)$): $\mathrm{MMD}^2$ (mean $\pm$ std over 30 runs)}
\label{tab:sweep_mmd2_sep4}
\begin{tabular}{l|cccc}
\toprule
$K \backslash N$ & N=100 & N=1000 & N=10000 & analytical \\
\midrule
K=1 & $0.0340 \pm 0.0180$ & $0.0202 \pm 0.0090$ & $0.0156 \pm 0.0090$ & $0.0119 \pm 0.0012$ \\
K=4 & $0.1261 \pm 0.0669$ & $0.1034 \pm 0.0566$ & $0.0736 \pm 0.0543$ & $0.0008 \pm 0.0003$ \\
K=16 & $0.2251 \pm 0.1118$ & $0.2212 \pm 0.1242$ & $0.1741 \pm 0.1179$ & $0.0001 \pm 0.0001$ \\
K=64 & $0.2814 \pm 0.1346$ & $0.2811 \pm 0.1461$ & $0.2483 \pm 0.1399$ & $-0.0000 \pm 0.0001$ \\
K=256 & $0.3058 \pm 0.1422$ & $0.3061 \pm 0.1527$ & $0.2891 \pm 0.1472$ & $-0.0000 \pm 0.0001$ \\
\bottomrule
\end{tabular}
\end{table}

\subsection{Gaussian mixture models (conditional model)} 

\begin{table}[H]
\centering
\caption{\textbf{GMM, In-Distribution} ($\mathrm{SW}_2$ mean $\pm$ std over 30 runs)}
\label{tab:score_sw2_in}
\begin{tabular}{l|cccc}
\toprule
$K \backslash N$ & N=100 & N=1000 & N=10000 & analytical \\
\midrule
K=1 & $1.2113 \pm 0.2110$ & $1.2470 \pm 0.1079$ & $1.1704 \pm 0.0914$ & $1.2133 \pm 0.0310$ \\
K=4 & $0.6678 \pm 0.2758$ & $0.2842 \pm 0.0728$ & $0.2318 \pm 0.0606$ & $0.2414 \pm 0.0308$ \\
K=16 & $0.7735 \pm 0.3478$ & $0.2206 \pm 0.0775$ & $0.1454 \pm 0.0567$ & $0.1044 \pm 0.0314$ \\
K=64 & $0.8130 \pm 0.3713$ & $0.2134 \pm 0.0730$ & $0.1478 \pm 0.0533$ & $0.0736 \pm 0.0266$ \\
K=256 & $0.8288 \pm 0.3826$ & $0.1985 \pm 0.0697$ & $0.1350 \pm 0.0539$ & $0.0771 \pm 0.0285$ \\
\bottomrule
\end{tabular}
\end{table}

\begin{table}[H]
\centering
\caption{\textbf{GMM, In-Distribution} ($\mathrm{MMD}^2$ mean $\pm$ std over 30 runs)}
\label{tab:score_mmd_in}
\begin{tabular}{l|cccc}
\toprule
$K \backslash N$ & N=100 & N=1000 & N=10000 & analytical \\
\midrule
K=1 & $0.0590 \pm 0.0272$ & $0.0457 \pm 0.0097$ & $0.0370 \pm 0.0088$ & $0.0392 \pm 0.0026$ \\
K=4 & $0.0576 \pm 0.0474$ & $0.0039 \pm 0.0022$ & $0.0016 \pm 0.0019$ & $0.0002 \pm 0.0002$ \\
K=16 & $0.0839 \pm 0.0702$ & $0.0045 \pm 0.0031$ & $0.0015 \pm 0.0020$ & $0.0000 \pm 0.0002$ \\
K=64 & $0.0935 \pm 0.0791$ & $0.0045 \pm 0.0032$ & $0.0016 \pm 0.0023$ & $-0.0000 \pm 0.0001$ \\
K=256 & $0.0977 \pm 0.0825$ & $0.0040 \pm 0.0028$ & $0.0017 \pm 0.0021$ & $-0.0000 \pm 0.0002$ \\
\bottomrule
\end{tabular}
\end{table}

\begin{table}[H]
\centering
\caption{\textbf{GMM, Out-of-Distribution} ($\mathrm{SW}_2$ mean $\pm$ std over 30 runs)}
\label{tab:score_sw2_ood}
\begin{tabular}{l|cccc}
\toprule
$K \backslash N$ & N=100 & N=1000 & N=10000 & analytical \\
\midrule
K=1 & $4.0336 \pm 0.5817$ & $4.0508 \pm 0.3100$ & $3.9606 \pm 0.2339$ & $4.2102 \pm 0.0441$ \\
K=4 & $2.7357 \pm 0.6423$ & $2.7265 \pm 0.4211$ & $2.4675 \pm 0.3453$ & $2.5081 \pm 0.0456$ \\
K=16 & $2.4276 \pm 0.7330$ & $2.4556 \pm 0.6946$ & $2.1718 \pm 0.4914$ & $1.7704 \pm 0.0374$ \\
K=64 & $2.6212 \pm 0.8565$ & $2.7205 \pm 0.9299$ & $2.4005 \pm 0.6674$ & $1.5531 \pm 0.0455$ \\
K=256 & $3.0238 \pm 1.0936$ & $3.0562 \pm 1.1639$ & $2.6706 \pm 0.8479$ & $1.4670 \pm 0.0506$ \\
\bottomrule
\end{tabular}
\end{table}

\begin{table}[H]
\centering
\caption{\textbf{GMM, Out-of-Distribution} ($\mathrm{MMD}^2$ mean $\pm$ std over 30 runs)}
\label{tab:score_mmd_ood}
\begin{tabular}{l|cccc}
\toprule
$K \backslash N$ & N=100 & N=1000 & N=10000 & analytical \\
\midrule
K=1 & $0.2616 \pm 0.0777$ & $0.2343 \pm 0.0382$ & $0.2262 \pm 0.0347$ & $0.2692 \pm 0.0077$ \\
K=4 & $0.2208 \pm 0.1082$ & $0.0892 \pm 0.0303$ & $0.0583 \pm 0.0202$ & $0.0362 \pm 0.0027$ \\
K=16 & $0.2514 \pm 0.1312$ & $0.1009 \pm 0.0450$ & $0.0610 \pm 0.0322$ & $0.0102 \pm 0.0009$ \\
K=64 & $0.2782 \pm 0.1413$ & $0.1539 \pm 0.0785$ & $0.0879 \pm 0.0554$ & $0.0067 \pm 0.0008$ \\
K=256 & $0.3068 \pm 0.1460$ & $0.2036 \pm 0.1022$ & $0.1133 \pm 0.0767$ & $0.0056 \pm 0.0008$ \\
\bottomrule
\end{tabular}
\end{table}

\subsection{Gaussian mixture models (separate models)}\label{subsec:sep2}

\begin{table}[H]
\centering
\caption{\textbf{GMM, In-Distribution} Separate models, ($\mathrm{SW}_2$ mean $\pm$ std over 30 runs)}
\label{tab:score_sep_sw2_in}
\begin{tabular}{l|cccc}
\toprule
$K \backslash N$ & N=100 & N=1000 & N=10000 & analytical \\
\midrule
K=1 & $1.2393 \pm 0.2418$ & $1.2108 \pm 0.1863$ & $1.2029 \pm 0.1275$ & $1.2157 \pm 0.0305$ \\
K=4 & $0.7108 \pm 0.2359$ & $0.4015 \pm 0.1510$ & $0.2925 \pm 0.1031$ & $0.2452 \pm 0.0317$ \\
K=16 & $0.8054 \pm 0.3097$ & $0.4442 \pm 0.2769$ & $0.2116 \pm 0.1222$ & $0.1128 \pm 0.0311$ \\
K=64 & $0.8440 \pm 0.3406$ & $0.5025 \pm 0.4534$ & $0.2196 \pm 0.1230$ & $0.0843 \pm 0.0251$ \\
K=256 & $0.8624 \pm 0.3564$ & $0.5514 \pm 0.6366$ & $0.2214 \pm 0.1321$ & $0.0740 \pm 0.0336$ \\
\bottomrule
\end{tabular}
\end{table}

\begin{table}[H]
\centering
\caption{\textbf{GMM, In-Distribution} Separate models, ($\mathrm{MMD}^2$ mean $\pm$ std over 30 runs)}
\label{tab:score_sep_mmd_in}
\begin{tabular}{l|cccc}
\toprule
$K \backslash N$ & N=100 & N=1000 & N=10000 & analytical \\
\midrule
K=1 & $0.0630 \pm 0.0308$ & $0.0444 \pm 0.0176$ & $0.0419 \pm 0.0119$ & $0.0396 \pm 0.0027$ \\
K=4 & $0.0622 \pm 0.0525$ & $0.0104 \pm 0.0110$ & $0.0053 \pm 0.0059$ & $0.0002 \pm 0.0002$ \\
K=16 & $0.0864 \pm 0.0743$ & $0.0183 \pm 0.0238$ & $0.0056 \pm 0.0099$ & $0.0000 \pm 0.0002$ \\
K=64 & $0.0969 \pm 0.0862$ & $0.0233 \pm 0.0323$ & $0.0064 \pm 0.0110$ & $0.0000 \pm 0.0002$ \\
K=256 & $0.1014 \pm 0.0921$ & $0.0269 \pm 0.0383$ & $0.0066 \pm 0.0127$ & $-0.0000 \pm 0.0001$ \\
\bottomrule
\end{tabular}
\end{table}

\begin{table}[H]
\centering
\caption{\textbf{GMM, Out-of-Distribution} Separate models, ($\mathrm{SW}_2$ mean $\pm$ std over 30 runs)}
\label{tab:score_sep_sw2_ood}
\begin{tabular}{l|cccc}
\toprule
$K \backslash N$ & N=100 & N=1000 & N=10000 & analytical \\
\midrule
K=1 & $4.2381 \pm 0.6608$ & $3.8951 \pm 0.3696$ & $3.8501 \pm 0.2915$ & $4.2093 \pm 0.0472$ \\
K=4 & $2.7454 \pm 0.7513$ & $2.3919 \pm 0.4998$ & $2.2963 \pm 0.2614$ & $2.5141 \pm 0.0447$ \\
K=16 & $2.2444 \pm 0.8758$ & $1.8197 \pm 0.6064$ & $1.6168 \pm 0.3447$ & $1.7899 \pm 0.0377$ \\
K=64 & $2.5258 \pm 2.1074$ & $1.6895 \pm 0.6533$ & $1.4262 \pm 0.3536$ & $1.5467 \pm 0.0364$ \\
K=256 & $2.7293 \pm 3.1365$ & $1.6850 \pm 0.7136$ & $1.3582 \pm 0.3864$ & $1.4619 \pm 0.0493$ \\
\bottomrule
\end{tabular}
\end{table}

\begin{table}[H]
\centering
\caption{\textbf{GMM, Out-of-Distribution} Separate models, ($\mathrm{MMD}^2$ mean $\pm$ std over 30 runs)}
\label{tab:score_sep_mmd_ood}
\begin{tabular}{l|cccc}
\toprule
$K \backslash N$ & N=100 & N=1000 & N=10000 & analytical \\
\midrule
K=1 & $0.2727 \pm 0.0964$ & $0.2484 \pm 0.0501$ & $0.2482 \pm 0.0373$ & $0.2706 \pm 0.0076$ \\
K=4 & $0.2555 \pm 0.1423$ & $0.1006 \pm 0.0552$ & $0.0559 \pm 0.0219$ & $0.0365 \pm 0.0025$ \\
K=16 & $0.3063 \pm 0.1723$ & $0.1165 \pm 0.0949$ & $0.0368 \pm 0.0221$ & $0.0106 \pm 0.0010$ \\
K=64 & $0.3283 \pm 0.1820$ & $0.1378 \pm 0.1121$ & $0.0368 \pm 0.0238$ & $0.0066 \pm 0.0007$ \\
K=256 & $0.3383 \pm 0.1888$ & $0.1585 \pm 0.1274$ & $0.0407 \pm 0.0273$ & $0.0056 \pm 0.0008$ \\
\bottomrule
\end{tabular}
\end{table}


\end{document}